\documentclass[lettersize,journal]{IEEEtran}
\usepackage{amsmath,amsfonts}
\usepackage{algorithmic}
\usepackage{algorithm}
\usepackage{array}
\usepackage[caption=false,font=normalsize,labelfont=sf,textfont=sf]{subfig}
\usepackage{textcomp}
\usepackage{stfloats}
\usepackage{url}
\usepackage{verbatim}
\usepackage{graphicx}
\usepackage{cite}
\usepackage{xcolor}
\usepackage{times}
\usepackage{epsfig}
\usepackage{amssymb}
\usepackage[group-separator={,},group-minimum-digits=4]{siunitx}
\usepackage{amsthm}
\usepackage{threeparttable}
\usepackage{multirow}
\newtheorem{theorem}{Theorem}[section]

\usepackage[normalem]{ulem}

\begin{document}

\title{Adaptive aggregation of Monte Carlo augmented decomposed filters for efficient group-equivariant convolutional neural network}

\author{Wenzhao Zhao,
	\and
	Barbara D. Wichtmann, \and
	Steffen Albert, \and
	Angelika Maurer, \and
	Frank G. Zöllner \and    
	and Jürgen W. Hesser
	
\thanks{Wenzhao Zhao is with School of Computer and Artificial Intelligence, Nanjing University of Finance and Economics, and Interdisciplinary Center for Scientific Computing, Mannheim Institute for Intelligent Systems in Medicine, Medical Faculty Mannheim, Heidelberg University. 
	Barbara D. Wichtmann is with Clinic of Neuroradiology, University Hospital Bonn, and German Center for Neurodegenerative Diseases (DZNE).
	Angelika Maurer is with Department of Diagnostic and Interventional Radiology, University Hospital Bonn.
	Steffen Albert and Frank G. Zöllner are with Computer Assisted Clinical Medicine, Mannheim Institute for Intelligent Systems in Medicine, Medical Faculty Mannheim, Heidelberg University.
	Jürgen W. Hesser is with Interdisciplinary Center for Scientific Computing,
	Central Institute for Computer Engineering,
	CSZ Heidelberg Center for Model-Based AI, Data Analysis and Modeling in Medicine, Mannheim Institute for Intelligent Systems in Medicine, Medical Faculty Mannheim, Heidelberg University.  E-mail: zhaowenzhaoyz@163.com.
}
\thanks{Manuscript received April 19, 2021; revised August 16, 2021.}}

\markboth{Journal of \LaTeX\ Class Files,~Vol.~14, No.~8, August~2021}%
{Shell \MakeLowercase{\textit{et al.}}: A Sample Article Using IEEEtran.cls for IEEE Journals}


\maketitle

\begin{abstract}
Group-equivariant convolutional neural networks (G-CNN) heavily rely on parameter sharing to increase CNN's data efficiency and performance. However, the parameter-sharing strategy greatly increases the computational burden for each added parameter, which hampers its application to deep neural network models. In this paper, we address these problems by proposing a non-parameter-sharing approach for group equivariant neural networks. The proposed methods adaptively aggregate a diverse range of filters by a weighted sum of stochastically augmented decomposed filters. We give theoretical proof about how the group equivariance can be achieved by our methods. Our method applies to both continuous and discrete groups, where the augmentation is implemented using Monte Carlo sampling and bootstrap resampling, respectively. 
Our methods also serve as an efficient extension of standard CNN. 
The experiments show that our method outperforms parameter-sharing group equivariant networks and enhances the performance of standard CNNs in image classification and denoising tasks, by using suitable filter bases to build efficient lightweight networks. 
The code is available at \url{https://github.com/ZhaoWenzhao/MCG_CNN}.

\end{abstract}


\section{Introduction}
\label{sec:intro}
\IEEEPARstart{E}{quivariance} or invariance under some transformation of input is a desired characteristic for many applications including computer vision tasks such as image classification and image denoising.
This is due to the widespread existence of transformations such as affine transformations in natural images as shown in Fig. \ref{fig:shear_cbsd}. For image classification, one may expect the classification algorithms to recognize the objects in the image regardless of any affine transformation, which is called "invariance". For image denoising, we expect the denoiser to process transformed input images and to output a corresponding transformed outputs, which is called "equivariance". Invariance can be considered a special case of "equivariance". We see that an affine equivariant deep learning model can have a good generalization and robustness against the affine transformation of the input image data. The deep learning methods that boost equivariance can therefore improve data efficiency of deep learning.
To achieve equivariant deep learning, researchers generally adopt two different approaches: data augmentation\cite{wang2022data,quiroga2020revisiting}, and group-equivariant network architectures\cite{kondor2018generalization}. Data augmentation is an efficient and popular method for enhancing a given model's group equivariance. However, there is no guarantee for group equivariance with respect to unseen data. 

\begin{figure}[t]
	\centering
	\includegraphics[width=0.7\linewidth]{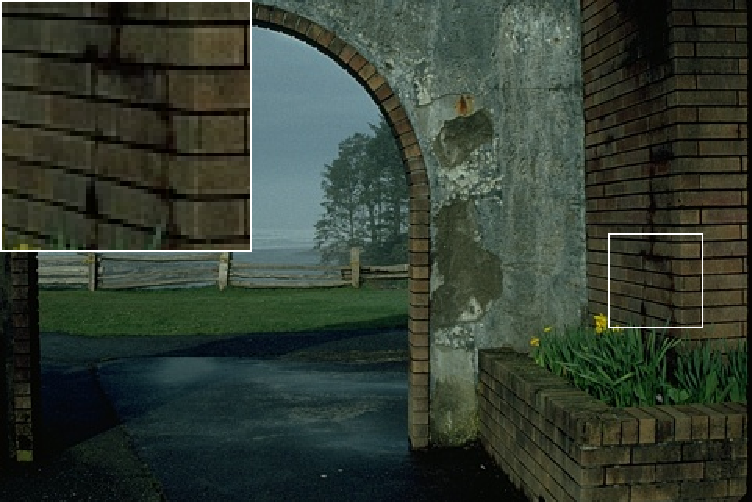}		
	\caption{ An example of affine transformation in real life. The image is from the CBSD432\cite{martin2001database} dataset. As shown in the white rectangular box, the horizontal lines of bricks undergo a shear transform along the vertical direction.}
	\label{fig:shear_cbsd}
\end{figure}

Group equivariant neural networks are another way to enhance group equivariant deep learning by focusing on the neural network's architecture itself.
Convolutional neural networks (CNNs), one of the most widespread deep neural network architectures in computer vision, show a desirable property of translation equivariance due to its "sliding window" strategy inspired by human vision\cite{fukushima1980neocognitron,lecun1989backpropagation}. In recent years, a sheer amount of publications have emerged aiming at developing and applying more advanced group equivariant CNNs to improve CNN's sample efficiency and generalizability\cite{kruger2023equivariant,lyle2019analysis,he2021efficient}. The concept of group equivariant CNN (G-CNN) was first proposed by Cohen and Welling in \cite{cohen2016group}, which exploited a higher degree of weight sharing by increasing the number of convolutional channels with the periodical rotation of the same convolutional kernel. This idea was further extended in \cite{cohen2016steerable} by introducing steerable filters which decomposed the convolutional kernel with an orthogonal basis of roto-reflection groups.

Following the work of rotation equivariant CNN, in recent years, there have been a lot of studies based on filter decomposition for exploring scale equivariant CNN\cite{sosnovik2019scale,sosnovik2021disco,sangalli2021scale,zhu2022scaling}, 
scale-rotation equivariant CNN\cite{gao2021deformation,he2021efficient,QIAO2025106980}, rigid transformation equivariant CNN\cite{jenner2022steerable}, and affine invaraint CNN\cite{shen_efficient_2024}. 
Attention mechanisms have been introduced in \cite{romero2020attentive,he2021efficient} to help better identify optimal filter banks and boost equivariance performance. In \cite{Wu_Liu_Sun_Yang_Dong_Lin_Tang_Mi_Jin_Wei_2025}, the concept of group bias is proposed to improve the performance of rotation equivariant CNN. The idea of group equivariance has also been introduced to transformer networks to improve the transformer's data efficiency.  Apart from filter decomposition, the feature alignment has also proven to be helpful for improving CNN's group equivariance against affine image transforms\cite{sunempowering}.

The existing works for filter-decomposition-based group equivariant CNN all require increasing channel numbers to increase parameter-sharing degree. This is because the existing works are based on group convolution\cite{kondor2018generalization}\cite{cohen2019general} and performing group convolution needs lifting or mapping the image data to the transformation group space, where additional dimensions are introduced and the integration or average along the additional dimensions is handled by the same learnable parameters (causing a higher parameter-sharing degree).  
This integration along the additional dimensions brings in a heavy computational burden\cite{kruger2023equivariant} and it causes an imbalance between the number of computational operations and the number of learnable parameters, which hence hampers their practical application to complex neural network architectures. 
Due to the computational burden needed for considering one kind of transform equivariance, the existing works of affine G-CNN are limited to transforms such as scaling, rotation, and reflection. So far, further including the shear transform is rarely considered in the conventional framework of affine G-CNN.

In addition, it has been shown that neural networks with greater depth and a larger number of parameters usually have better generalization performance\cite{yang2020rethinking,nakkiran2021deep}. The heavy computational burden of a single group equivariant layer makes it difficult to apply parameter-sharing G-CNN to large neural network models. In this work, we show that the proposed efficient non-parameter-sharing G-CNNs can achieve superior performance to parameter-sharing G-CNNs when combined with advanced neural network architectures.

In this paper, we propose an efficient implementation of non-parameter-sharing G-CNNs based on an adaptive aggregation of Monte Carlo augmented decomposed filters. The contribution of this paper is embodied in four aspects:
\begin{itemize}
	\item We propose an efficient non-parameter-sharing group equivariant network, where no additional channels or dimensions are introduced in implementing group convolution. Our method serves as an efficient extension of standard CNN. We give theoretical proof of how the group equivariance is achieved with conventional neural network training.  
	\item Thanks to the convenience of weighted Monte Carlo (MC) sampling in implementation, our work can consider a more flexible mix of different simple transforms, we thereby introduce shear transform for affine G-CNN and demonstrate its potential to improve G-CNNs' performance on natural images.
	\item Our non-parameter-sharing G-CNNs achieve superior performance to parameter-sharing-based G-CNNs when combined with advanced neural network architectures. Our approach does not increase the computation burden and achieves high parameter and data efficiency compared with standard CNNs. 
	\item With a set of suitable filter bases, the proposed networks serve as promising alternatives to standard CNNs for both image classification and image denoising tasks. Compared with standard CNNs, the proposed methods are good at exploiting large convolutional kernels efficiently, which helps build an efficient lightweight image-denoising network.
\end{itemize}

The paper is organized as follows: In the Methods section, we review the general framework of the group-equivariant model and introduce the details of our approach. We show the experimental results and discussions in the Experiments section and conclude the paper in the Conclusion section.

\section{Methods}
\label{sec:method}
\subsection{The general framework of group-equivariant model}
Group\cite{dummit2004abstract} is a classical mathematical concept, which is defined to be a set with a corresponding binary operation that is associative, contains an identity element, and has an inverse element for each element. In this paper, all the discussed groups are assumed to be locally compact groups. 
	Following \cite{kondor2018generalization}, we will briefly introduce the definition of group equivariant mapping and group convolution. 

\subsubsection{Group equivariance }

In this paper, we consider a group $G$ for the affine transformations on 2D images $\mathbb{R}^2$, which can be written as $G=\mathbb{R}^2\rtimes \mathcal{A}$, a semidirect product between the translation group $\mathbb{R}^2$ and another affine transform group $\mathcal{A}$ (whose group element for 2D images takes the representation of a $2\times 2$ matrix). Its group product rule is defined as 
\begin{equation}
	\begin{array}{l}
		g_1\bullet g_2 = (x_1,a_1)\bullet (x_2,a_2)\\
		=(x_1+M(a_1) x_2, a_1+a_2),
	\end{array}
	\label{eq:product}
\end{equation}
where "$\bullet$" denotes the group product operator, $g_1 = (x_1,a_1)$, $g_2 = (x_2,a_2)$ with $x_1,x_2\in \mathbb{R}^2$, $a_1,a_2\in \mathbb{R}^4$, and function $M: \mathbb{R}^4\rightarrow \mathcal{A}$. In this paper, we consider the following affine group, in particular, for any $a = (\alpha, \sigma, s, r)$ with $\alpha, \sigma, s \in \mathbb{R}$, $M(a) = R(\theta)A(\alpha)S_1(s)S_2(r)$, where
\begin{equation}
	S_1(s) = 
	\begin{bmatrix}
		1 & s\\
		0 & 1
	\end{bmatrix},
	\label{eq:shear}
\end{equation}

\begin{equation}
	S_2(r) = 
	\begin{bmatrix}
		1 & 0\\
		r & 1
	\end{bmatrix},
	\label{eq:shear2}
\end{equation}

\begin{equation}
	A(\alpha) = 
	\begin{bmatrix}
		2^{\alpha} & 0\\
		0 & 2^{\alpha}
	\end{bmatrix},
	\label{eq:scaling}
\end{equation}	

\begin{equation}
	R(\theta) = 
	\begin{bmatrix}
		\cos{\theta} & \sin{\theta}\\
		-\sin{\theta} & \cos{\theta}
	\end{bmatrix}.
	\label{eq:rotation}
\end{equation}	
It should be noted that the existing works on affine G-CNN only consider translation, scaling, rotation, and mirror transforms. In this work, shear transform is included to form a more general case and explore its potential for boosting G-CNN's performance on natural images.

For a group element of the affine transformation group $g\in G$, there is a corresponding group action on an index set $\mathcal{X}$, i.e., a transformation $T: G\times \mathcal{X}\rightarrow \mathcal{X}$ for the index set. And for any $g_1, g_2 \in G$ and $x\in \mathcal{X}$, we have
\begin{equation}
	T(g_1\bullet g_2,x) = T(g_1,T(g_2,x)).
	\label{eq:t_product}
\end{equation}	
The corresponding transformation $\mathbb{T}_g$ for any function $f: \mathcal{X}\rightarrow \mathbb{C}$ can be further defined as $\mathbb{T}_g: f\rightarrow f^\prime$ where $f^\prime (T(g,x))=f(x)$.

With the concept of group and group actions, we can now define the group equivariant map. 
Suppose we have a function $f:\mathcal{X}\rightarrow V$ to be the input image or feature map of a neural network layer with $V$ as a vector space.
Let $L_V(\mathcal{X})$ denote the Banach space of functions ${f:\mathcal{X}\rightarrow V}$.
Consider a map $\phi: L_{V_1}(\mathcal{X}_1)\rightarrow L_{V_2}(\mathcal{X}_2)$ between two function spaces $L_{V_1}(\mathcal{X}_1):\{f: \mathcal{X}_1\rightarrow V_1\}$ and $L_{V_2}(\mathcal{X}_2):\{f: \mathcal{X}_2\rightarrow V_2\}$.
For $g\in G$, we have $T_g$ and $T^{\prime}_g$ to be G actions corresponding to set $\mathcal{X}_1$ and $\mathcal{X}_2$, as well as $\mathbb{T}_g$ and $\mathbb{T}^{\prime}_g$.
The map $\phi$ is group equivariant if and only if 
\begin{equation}
	\forall g\in G, \phi(\mathbb{T}_g(f))=\mathbb{T}^{\prime}_g(\phi(f)) 
	\label{eq:equivariant}
\end{equation}	 

\subsubsection{Group convolution}

A standard convolution of a function $f$ with a kernel $\psi$$: \mathbb{R}\rightarrow \mathbb{R}$, is a translation-equivariant map, which can be written as
\begin{equation}
	(\psi*f)(x) = \int \psi(-x+x^\prime)f(x^\prime)dx^\prime,
	\label{eq:conv_continuous}
\end{equation}

Group convolution is a generalization of standard convolution by introducing the group operation. 
The group convolution \cite{kondor2018generalization}\cite{cohen2019general}\cite{bekkers2019b}\cite{he2021efficient} on a compact group $G$ at group element $g$ is written as
\begin{equation}
	(\psi*f)(g)=\int_{G}\psi(g^{-1}\bullet g^\prime)f(g^\prime)d\mu (g^\prime)
	\label{eq:equivariant_conv}
\end{equation}	
where $\mu$ is the Haar measure, and $f,\psi: G\rightarrow \mathbb{C}$.
It should be noted that plain convolution is a special case of group convolution when only the translation group is considered (i.e., $g^{-1} = -x$; $g'=x'$ and the "$\bullet$" corresponds to "$+$").
\cite{kondor2018generalization} proved that the group convolution defined in equation \eqref{eq:equivariant_conv} is a group-equivariant map.
The group convolution shows favorable group equivariance in lots of research works\cite{gao2021deformation,sosnovik2019scale}. 
The group convolution can be considered to be a mapping that uses a transformation-level averaging procedure to enhance the transformation equivariance or invariance of the mapping. Therefore, in this paper, we assume the general effectiveness of the group convolution operator for any kind of geometric group and its common subsets.

\subsection{Adaptive aggregation of Monte Carlo augmented decomposed filters}

In a discrete implementation of group convolution, the numerical integration is usually implemented based on the trapezoidal rule\cite{atkinson1991introduction} using evenly sampled group elements $g^\prime$ in equation \eqref{eq:equivariant_conv}. 
For each input feature map channel (when considering many different kinds of affine transforms such as scaling, rotation, and mirror), nested integrals are needed, i.e. one nested integral per transform is considered. By this, the approach increases the computation burden exponentially with the number of considered transforms, which leads to the curse of dimensionality\cite{weinzierl2000introduction}. For example, when we have $m$ different elements per transform and $n$ transforms, this amounts to $m^n$ terms to be evaluated.

To improve the flexibility of group convolution for the general affine transform group and avoid the curse of dimensionality, in this work, we propose to approximate the multi-dimensional integral over group operations in the group convolution by MC integration.

\subsubsection{Monte Carlo integration}

MC integration is known to tackle high-dimensional integration with robust convergence independent of the number of dimensions\cite{weinzierl2000introduction}. We consider for brevity only the standard MC variant, being aware that more efficient schemes such as Quasi-MC have the potential to substantially increase the performance further\cite{caflisch1998monte,lepage1978new}.

For a multi-dimensional Monte Carlo integral, we have the theorem \cite{przybylowicz2022foundations,kong2003theory,kiria2016calculation} as follows,
\begin{theorem}
	Let $\mu_p$ be a probabilistic measure on $(\mathbb{R}^d,\mathcal{B}(\mathbb{R}^d))$, i.e., $\mu_p(\mathbb{R}^d)=1$, and $\mathcal{B}(\mathbb{R}^d)$ denotes the Borel algebra on $\mathbb{R}^d$ with $d$ the number of dimensions. For $f\in L^2(\mathbb{R}^d,\mathcal{B}(\mathbb{R}^d),\mu_p)$, we define 
	\begin{equation}
		I(f) = \int_{\mathbb{R}^d}f(x)d\mu_p(x), 
		\label{eq:mcintegration}
	\end{equation}
	and
	\begin{equation}
		Q_N(f) = \frac{1}{N}\sum_{i=1}^{N}f(\xi_i), 
		\label{eq:mc}
	\end{equation}
	where $(\xi_i)_{i\in N}$ is an i.i.d sequence of random variables with distributions $\mu_p$.
	We have $Q_N(f)\rightarrow I(f)$ when $N\rightarrow +\infty$. For all $N\in \mathbb{N}$, there is 
	\begin{equation}
		(\mathbb{E}\|I(f)-Q_N(f)\|^2)^{1/2} =\sigma(f)/\sqrt{N},
		\label{eq:mcerror}
	\end{equation}
	where $\sigma^2(f)=I(f^2)-(I(f))^2$, and $\|\cdot \|$ is the $l^2$ norm.
	
	\label{theorem:mc-int}
\end{theorem}

The Haar measure in (\ref{eq:equivariant_conv}) can be considered to be a corresponding probabilistic measure $\mu_p$.
Therefore, it is theoretically justified to apply MC sampling for the discrete implementation of G-CNN.

\subsubsection{Discrete implementation of G-CNN with MC integration}
\label{sec:discreteMCG}
In the discrete implementation, we stochastically sample the group operations including, in our example, scaling, rotation, and shear transform. This approach allows a more flexible choice of the number of used transformations and decouples the relationship between the number of output channels and the number of categories of considered transformations.

Specifically, when we consider a filter $W = w \cdot \psi$ with a fixed base filter $\psi$ and $w$ the trainable scalar weight, a continuous CNN layer can be written as
\begin{equation}
	\begin{array}{l}
		f^{(l+1)}_{c_o}(x) = \sum_{c_i} w^{(l)}_{c_o,c_i} (\psi*f^{(l)}_{c_i})(x)\\
		= \sum_{c_i}\int_{\mathbb{R}^2} w^{(l)}_{c_o,c_i} \psi(u-x)f^{(l)}_{c_i}(u)du
	\end{array}
	\label{eq:continu_conv}
\end{equation}	

A corresponding discrete implementation of the convolutional layer\footnote{It should be noted that in this paper, for simplicity, we omit point-wise nonlinearity functions, constant scalar coefficients, and normalization layers in neural networks, which do not affect the group equivariance\cite{kondor2018generalization}.} of $l$-th layer is as below
\begin{equation}
	f^{(l+1)}_{c_o}(x) = \sum_{c_i} \sum_{u} w^{(l)}_{c_o,c_i}\psi(u-x) f^{(l)}_{c_i}(u)
	\label{eq:cnn}
\end{equation}	 
where $x,u\in \mathbb{R}^2$, $\psi(\cdot)$ denotes the spatial convolutional filter function with a domain of translation group $\mathbb{R}^2$, $c_i \in [1,C_l]$ and $c_o\in [1,C_{l+1}]$. $f^{(l)}_{c_i}(x)$ is the feature map of the $l$-th layer and $w^{l}_{c_o,c_i}$ is the filter weight for the filter of the $l$-th layer with output channel $c_o$ and input channel $c_i$.

A continuous affine group equivariant CNN can be written as 
\begin{equation}
	\begin{array}{l}
		f^{(l+1)}_{c_o}(g) = \sum_{c_i} w^{(l)}_{c_o,c_i} (\psi * f^{(l)}_{c_i})(g) \\
		= \sum_{c_i}\int_{G}w^{(l)}_{c_o,c_i}\psi(g^{-1}\bullet g^\prime)f^{(l)}_{c_i}(g^\prime)d\mu (g^\prime)
	\end{array}
	\label{eq:equivariant_conv_rd}
\end{equation}

Let $g = (x,a)$ and $g^\prime = (u,b)$, we can rewrite the Haar integration in a group convolution of the $l$-th layer as:
\begin{equation}
	\begin{array}{l}
		f^{(l+1)}_{c_o}(x,a) 
		= \sum_{c_i}\int_{\mathbb{R}^4}\int_{\mathbb{R}^2} w^{(l)}_{c_o,c_i}2^{-2\alpha_b}\cdot \\
		\psi(-x+M(-a)u,-a+b)f^{(l)}_{c_i}(u,b)du db
	\end{array}
	\label{eq:equivariant_conv_rd}
\end{equation}	
where we have the transform parameter vectors $a = [\alpha_a,\theta_a,s_a,r_a]$, and $b = [\alpha_b,\theta_b,s_b,r_b]$.

A typical corresponding discrete G-CNN can be written as below:
\begin{equation}
	\begin{array}{l}
		f^{(l+1)}_{c_o}(x,a) =\sum_{c_i} \sum_{b} \sum_{u} 
		w^{(l)}_{c_o,c_i} 2^{-2\alpha_b} \cdot \\\psi(-x+M(a)u,-a+b) f^{(l)}_{c_i}(u,b) 
	\end{array}
	\label{eq:gcnn}
\end{equation}	
In particular, the sum over the parameter vector $b$ is a three-layer nested sum corresponding to the nested integrals in the continuous domain, which, as mentioned in previous sections, leads to a heavy computational burden.

The Monte-Carlo integration considers $a$ and $b$ as random variables. Suppose their entries $\alpha=\xi_{\alpha}$, $\theta=\xi_{\theta}$, $s=\tan(\xi_{s})$, and $r=\tan(\xi_{r})$, where   $\xi_{\alpha}$, $\xi_{\theta}$, $\xi_{s}$ and $\xi_{r}$ are uniformly distributed in the range of $[\eta_{\alpha}^1,\eta_{\alpha}^2)$, $[-\eta_{\theta},\eta_{\theta})$, $[-\eta_s,\eta_s)$, and $[-\eta_r,\eta_r)$, respectively.

Suppose we draw $N^\prime$ samples of $a$, and $N$ samples of $b$, respectively. The nested sum over $b$ collapses into a one-dimension sum over $N$ samples for MCG-CNN (Monte Carlo Group-equivariant CNN):
\begin{equation}
	\begin{array}{l}
		f^{(l+1)}_{c_o}(x,a_{n^\prime}) =\sum_{c_i} \sum_{n} \sum_{u} w^{(l)}_{c_o,c_i}2^{-2\alpha_{b_{n}}}\cdot 
		\\ \psi(-x+M(-a_{n^\prime})u,-a_{n^\prime}+b_n) f^{(l)}_{c_i}(u,b_n)
		
	\end{array}
	\label{eq:gcnn}
\end{equation}	
where $n^{\prime}\in \{1,\dots, N^{\prime}\}$, and $n\in \{1,\dots, N\}$.

\subsubsection{Adaptive aggregation of MC-augmented filters}

The Monte-Carlo approximation of G-CNN allows a flexible choice of the number of sampling points $N$ per trainable weight $w^{(l)}$ independent of the number of dimensions. However, compared with standard CNN, the computational burden of MCG-CNN is still $N$ times larger. To eliminate the difference in computational burden between MCG-CNN and standard CNN, we propose WMCG-CNN (Weighted Monte Carlo Group-equivariant CNN)\footnote{The word "Weighted" in WMCG-CNN is used to emphasize that the number of trainable filter weights becomes transformation-wise in WMCG-CNN, which is thus an adaptive aggregation of augmented filters.}, which reduces the number of transformations per input feature map channel (also per trainable weight) $N$ to $1$ and uses filter-weight-wise sampling instead.
Specifically, we establish a one-to-one relationship between $b$, $c_o$ and $c_i$, as well as $a$ and $c_o$ by using $c_o$ and $c_i$ to index $a$ and $b$. Thus we introduce notation $b_{c_o,c_i}$ and $a_{c_o}$.

In this way, we yield WMCG-CNN with the equation \eqref{eq:gcnn} simplified into:
\begin{equation}
	\begin{array}{l}
		f^{(l+1)}_{c_o}(x,a_{c_o}) =\sum_{c_i} \sum_{u} 
		w^{(l)}_{c_o,c_i} 2^{-2\alpha_{b_{c_o,c_i}}} \cdot \\
		\psi(-x+M(-a_{c_o})u,-a_{c_o}+b_{c_o,c_i})
		f^{(l)}_{c_i}(u,b_{c_o,c_i}),
	\end{array}
	\label{eq:wmcgcnn}
\end{equation}	
WMCG-CNN allows us to significantly increase the number of used transformations without increasing the computational burden, which, as shown in the later experiments, helps WMCG-CNN achieve superior performance to traditional discrete G-CNN.

However, due to the changes happening to WMCG-CNN, a question arises, i.e., under which circumstances, the WMCG-CNN can still be analogous to continuous G-CNN as the discrete G-CNN does? Below, we show that random initialization of the trainable weights can help the WMCG-CNN to be analogous to continuous G-CNN.

\begin{theorem}
	Let $f^{(l)}$ be an input feature map of the $l$-th layer with the number of channels $C_l$, and for each channel the number of spatial sampling points along vertical direction $N_H$, the number of spatial sampling points along horizontal direction $N_W$.
	A WMCG-CNN layer is group equivariant when the width of CNN, $C_l \rightarrow \infty$, $N_H\rightarrow \infty$, $N_W\rightarrow \infty$, and there exists a constant $C<+\infty$ so that $\|\int_{\mathbb{R}} w d\mu_w(w)\|<C$ with $\mu_w$ a probabilistic measure on $(\mathbb{R},\mathcal{B}(\mathbb{R}))$ for the filter weight $w$, being a random variable.
	
	\label{theorem:wmcg}
\end{theorem}

\begin{proof}

	To prove the theorem, we have two steps: First, we construct a weighted integration function $I$ and prove it is group equivariant. Then, we show that equation \eqref{eq:wmcgcnn} corresponds to the discrete form of $I$.
	
	1) Given $g = (x,a_{c_o})$ and $g^\prime = (u,b)$, we define the integration on $\mathbb{R}\times G$ as
	\begin{equation}
		\begin{array}{l}
			I(x,a_{c_o}) \\
			= \int_{\mathbb{R}\times G} w\cdot \psi(g^{-1} \bullet g^\prime)f^{(l)}(g^\prime)d\mu (g^\prime)d\mu_w(w) \\
			= \int_{\mathbb{R}}\int_{\mathbb{R}^d} \int_{\mathbb{R}^2} w \psi(-x+M(-a_{c_o})u,-a_{c_o}+b) \\f^{(l)}(u,b)dudbdw
		\end{array}
		\label{eq:wmcg}
	\end{equation}	
	
	Since $\|\int_{\mathbb{R}} w d\mu_w(w)\|<C$, we have the constant $C_w = \int_{\mathbb{R}} w d(w)$.
	Thus 
	\begin{equation}
		\begin{array}{l}
			I(x,a_{c_o}) = C_w\cdot \int_{\mathbb{R}^d} \int_{\mathbb{R}^2} \psi(-x+M(-a_{c_o})u,-a_{c_o}+b) \\f^{(l)}(u,b)dudb
		\end{array}
		\label{eq:wmcg_c}
	\end{equation}	
	which is group equivariant.
	
	2) Let $q(x,a_{c_o},b)=\int_{\mathbb{R}^2}\psi(-x+M(-a_{c_o})u,-a_{c_o}+b) f^{(l)}(u,b)du$, so we have
	\begin{equation}
		\begin{array}{l}
			I(x,a_{c_o}) = \int_{\mathbb{R}}\int_{\mathbb{R}^d}wq(x,a_{c_o},b) dbdw
		\end{array}
		\label{eq:wmcg_q}
	\end{equation}

	Now, we consider the transition from continuous to discrete formulations.
	Since both $w$ and $b$ are independently randomly sampled with the samples indexed by $c_i$.
	According to Theorem \ref{theorem:mc-int}, we have 
	\begin{equation}
		\begin{array}{l}
			I(x,a_{c_o}) =\lim_{C_l\rightarrow \infty} \frac{1}{C_l}\sum_{c_i}w^{(l)}_{c_o,c_i}q(x,a_{c_o},b_{c_o,c_i}) 
		\end{array}
		\label{eq:wmcg1}
	\end{equation}	
	
	Since $u$ is sampled based on the trapezoidal rule, we have
	\begin{equation}
		\begin{array}{l}
			q(x,a_{c_o},b_{c_o,c_i})\\
			= \lim_{N_H\rightarrow +\infty}\lim_{N_W\rightarrow +\infty}\frac{1}{N_H N_W}\sum_{u}  2^{-2\alpha_{b_{c_o,c_i}}}\cdot\\ 
			\psi(-x+M(-a_{c_o})u,-a_{c_o}+b_{c_o,c_i}) f^{(l)}_{c_i}(u,b_{c_o,c_i}),
		\end{array}
		\label{eq:wmcg2}
	\end{equation}

	Meanwhile, we rewrite the corresponding convolution part of WMCG-CNN equation \eqref{eq:wmcgcnn} as  
	\begin{equation}
		\begin{array}{l}
			f^{(l+1)}_{c_o}(x,a_{c_o}) 
			=\frac{1}{C_l N_H N_W}\sum_{c_i} \sum_{u} w^{(l)}_{c_o,c_i} 2^{-2\alpha_{b_{c_o,c_i}}} \cdot\\
			\psi(-x+M(-a_{c_o})u,-a_{c_o}+b_{c_o,c_i}) f^{(l)}_{c_i}(u,b_{c_o,c_i}),
		\end{array}
		\label{eq:wmcg_f}
	\end{equation}	
	where $c_i\in \{ 1,2,\dots,C_l\}$, $u=(u_1,u_2)$ with $u_1\in \{1,2,\dots,N_H\}$ and $u_2\in \{1,2,\dots,N_W\}$. Here we include coefficient $\frac{1}{C_l N_H N_W}$ so that $f^{(l+1)}_{c_o}$ is the average of the samples.
	
	Therefore, by combining (\ref{eq:wmcg1}) and (\ref{eq:wmcg2}), we have
	\begin{equation}
		\begin{array}{l}
			I(x,a_{c_o}) \\
			=\lim_{C_l\rightarrow \infty} \lim_{N_H\rightarrow +\infty}\lim_{N_W\rightarrow +\infty}f^{(l+1)}_{c_o}(x,a_{c_o})
		\end{array}
		\label{eq:wmcg_disc}
	\end{equation}	
	The proof is completed.

\end{proof}
As we know, random initialization of trainable weights is a common strategy adopted in most existing state-of-the-art deep learning methods. Theorem \ref{theorem:wmcg} proves that the random weight initialization strategy together with the MC-augmented filters can help raise the CNN to a good starting point before training with an optimization algorithm, which therefore makes it easier for the network to find the optimal solution. This starting point is a network that approximately satisfies convolutional-layer-wise group equivariance.
Obviously, a necessary condition of an optimal solution is that in contrast to the approximate convolutional-layer-wise group equivariance, it is at least at the level of the entire neural network that the group equivariance is achieved approximately.

From Theorem \ref{theorem:mc-int}, we know that the convergence speed of Monte Carlo integration is slow. When the number of samples is small, the variance may not be satisfactory.  
However, with the weight $w$ as learnable parameters and the samples of transformations fixed, the neural network can learn optimal weight distribution to improve the group equivariance, which will be shown in the later experiments (Fig. \ref{fig:mGE}). 
Accordingly, we can rewrite equation \eqref{eq:wmcg1} as follows:
\begin{equation}
	\begin{array}{l}
		I(x,a_{c_o}) =\lim_{C_l\rightarrow \infty} \frac{1}{C_l}\sum_{c_i}w^{(l)}_{c_o,c_i}q(x,a_{c_o},b_{c_o,c_i}) \\
		\approx \hat{I}(x,a_{c_o}) = \frac{1}{C_l}\sum_{c_i}w^{(l)}_{c_o,c_i}q(x,a_{c_o},b_{c_o,c_i}) 
	\end{array}
	\label{eq:imp_smpl}
\end{equation}

The estimator $\hat{I}$ in equation \eqref{eq:imp_smpl}) has a similar form to that of the conventional importance sampling method\cite{glynn1989importance}. 
The differences are that the weight distribution in WMCG-CNN is not manually designed but is learned by iterative data-driven optimization algorithms for neural networks instead, and the learned weights may not be considered being i.i.d. random.

Assuming $I$ a unique constant for the given $x$ and $a_{c_o}$, we have the variance $Var(\hat{I}-I)=Var(\hat{I})$. 
	We know that the variance of the estimator $\hat{I}$ is reduced during the training when the training process leads to affine equivariance with $Var(\hat{I}-I)\rightarrow 0$. Such a training process usually requires rich affine transformations existing in the image sets. This requirement is usually met in practice because the affine transformations are common in natural images. Moreover, affine-transformation-based data augmentation is also a common training technique to achieve this.

In addition to factors during training, in practice, certain fundamental structural features of neural networks also help ensure training stability and low variance. Specifically, when applying WMCG-CNN, multiple bases are often used. This increases the diversity of the filters. Furthermore, residual connections are used in each bottleneck block of the hidden layer to avoid information loss and instability, which is common in most modern neural network architectures. Residual connections also implicitly increase the number of samples, as they transmit the same input to the next layer to be processed by more filters. Therefore, when using our method on the hidden layers of neural networks with residual connections, we generally do not need to worry about its stability.

\subsubsection{Filter decomposition and the relationship to traditional CNN filters}
In the previous section, we only consider one basis filter function $\psi$, to increase the expressiveness of networks, we adopt the filter decomposition approach to build convolutional filters by the sum of multiple weighted filter bases. 
Specifically, we have $W^{(l)}_{c_o,c_i}(x,a) = \sum_j w^{(l)}_{c_o,c_i,j} \Tilde{\psi}_j(x,a) $ with $\Tilde{\psi}_j(x,a)$ an orthogonal basis function with $x\in \mathbb{R}^2$ and $a\in \mathbb{R}^4$ the transform parameter vector, $w^{(l)}_{c_o,c_i,j}$ the trainable weights, $j\in [1,K]$, and $K$ the chosen number of basis functions. In the proposed WMCG-CNN, according to equation \eqref{eq:wmcgcnn} the WMCG-CNN can be written in a similar way to the standard CNN in equation \eqref{eq:cnn} as below:
\begin{equation}
	\begin{array}{l}
		f^{(l+1)}(c_o,x,a_{c_o}) =\sum_{c_i} \sum_{u} 2^{-2\alpha_{b_{c_o,c_i}}} W^{(l)}_{c_o,c_i}(\\ 
		-x+M(-a_{c_o})u,-a_{c_o}+b_{c_o,c_i}) f^{(l)}_{c_i}(u,b_{c_o,c_i})
		,
	\end{array}
	\label{eq:mcgcnn2}
\end{equation}	
In the practical discrete implementation, the choice of filter basis can be various. For different datasets and different tasks, the optimal filter basis can be different. In the following experiments, we generally adopt two kinds of filter basis: the Fourier-Bessel (FB) basis\cite{qiu2018dcfnet}, and the continuous Mexican hat (MH) wavelet basis\cite{ryan1994ricker,antoine1993image}. The scaling, rotation, and shear transformations are used to augment the FB filters. The MH filters are augmented by scaling, translation, and shear transformations.
	Supposing any filter basis is a matrix of size $k\times k$, for FB basis, we can have $k^2-1$ non-constant bases and a constant scalar basis at the most.

The 2D MH filters can be written as
\begin{equation}
	\begin{array}{l}
		\psi(\sigma_x,\sigma_y,x,y)=\frac{1}{2\pi \sigma_x\sigma_y}[2-\frac{x^2}{\sigma_x^2}-\frac{y^2}{\sigma_y^2}]e^{-\frac{x^2}{2\sigma_x^2}-\frac{y^2}{2\sigma_y^2}}
	\end{array}
	\label{eq:mexicanhat}
\end{equation}	
The peak frequency of the MH wavelet function $f_x=\frac{1}{\sqrt{2}\pi\sigma_x}$ and $f_y=\frac{1}{\sqrt{2}\pi\sigma_y}$\cite{wang2015frequencies}, which can be used for scaling along $x$ or $y$ axis.

It should be noted that when using the bases consisting of translation-augmented discrete Dirac delta functions, the proposed methods fall back into standard CNN filters. Figure \ref{fig:filterbasis} shows examples of different filter bases.
 
In addition, theoretically, in the training phase, the computational burden of the proposed method is slightly higher than that of the corresponding standard CNN that uses the same size of the convolutional kernel (as shown in Table \ref{tab:time}). This is because of the weighted sum of filter bases. Yet, the weighted sum can be pre-calculated with the results stored in memory before doing an inference. Thus, in the inference phase, the network using the proposed method has exactly the same computational complexity as the corresponding standard CNN.

\begin{table*}[htbp]
	\centering
	\caption{The computational time and memory usage on GPU Nvidia A100 40GB for CNN models trained on ImageNet dataset, where both CNN and our WMCG-CNN models are using the same kernel size and the same number of learnable parameters.}
	\centerline{\resizebox{14.0cm}{!}{ 
			\begin{threeparttable}[b]
\begin{tabular}{l|cc|cc|cc}
	\hline
	Base model                   & \multicolumn{2}{c|}{ResNeXt50} & \multicolumn{2}{c|}{ResNet50} & \multicolumn{2}{c}{ConvNeXt-S} \\ \hline
	Method                       & CNN          & WMCG-CNN       & CNN         & WMCG-CNN       & CNN          & WMCG-CNN        \\\hline
	Training Time (s/epoch)       & 3884         & 3914           & 2793        & 2808           & 4739         & 4754            \\
	Training VRAM Footprint (MB) & 35078        & 35116          & 24934       & 25106          & 25742        & 25990          \\\hline
\end{tabular}
	\end{threeparttable}}}
	\label{tab:time}
\end{table*}	

\begin{figure}[t]
	\centering
	\includegraphics[width=0.9\linewidth]{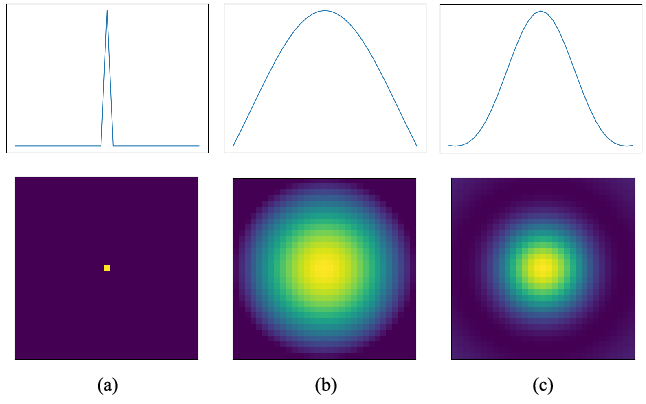}		
	\caption{Examples of filter bases in 1-dimension and 2-dimension space. (a) the discrete Dirac delta basis, (b) the Fourier Bessel basis; (c) the Mexican hat basis.}
	\label{fig:filterbasis}
\end{figure}

\subsection{Extending to discrete groups with bootstrap resampling}
In the previous sections, we focus on continuous groups and how the weighted G-CNN can be approximated with the discretized implementation assuming an infinite number of filter samples. However, our method can also apply to cases where the number of available group elements is far less than the number of input-output channel pairs (in equation \eqref{eq:wmcgcnn}). We can use the bootstrap resampling\cite{hesterberg2015teachers} method to make the number of augmented basis samples large enough to match each weighted input-output channel pair. 

\subsection{Integrating WMCG-CNN into the existing state-of-the-art CNN architectures}
\begin{figure}[t]
	\centering
	\includegraphics[width=0.9\linewidth]{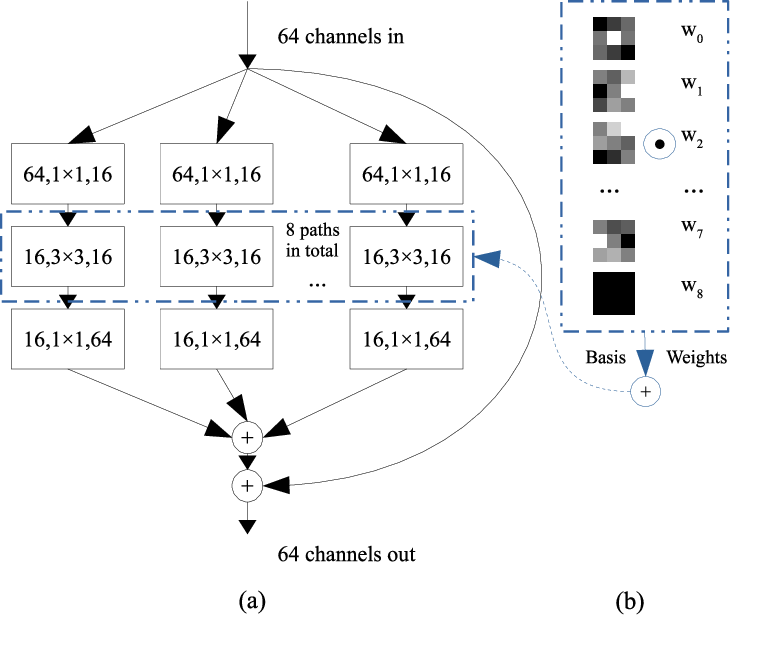}		
	\caption{ Integrating the proposed WMCG-CNN into the classic bottleneck architecture.
		(a) The example bottleneck block with group convolution using $3\times 3$ filters; (b) An example of filter composition with MC-augmented basis.}
	\label{fig:shearfilter}
\end{figure}

We see that when $\Tilde{\psi}_j$ degenerates to a scalar, i.e. a $1\times 1$ base filter, the convolution is obviously exactly group equivariant, while on the other hand, the non-scalar filter $\Tilde{\psi}_j$ requires a huge number of sampling points to approximate the continuous G-CNN. 
To leverage the advantage of $1\times 1$ base filters, one can add $1\times 1$-filter-based convolution layers as a secondary adaptive aggregation of features from the output channels of WMCG-CNN. 
By combining the $1\times 1$ layer with the $k\times k$ convolution layer into a single unit or block, the total number of considered transformations is increased from $C_l$ to $C_{l+1}C_l$ (i.e., the number of all the $k\times k$ filters used in the $l$-th layer) with a relatively small increase of the number of parameters. 
In addition, the $1\times 1$ CNN layer also helps to enrich the design space for WMCG-CNN, where the use of the small $1\times 1$ kernel helps to achieve a high parameter efficiency given the same level of expressiveness and the same number of parameters\cite{he2016deep}.

Interestingly, the secondary aggregation with a cascaded $1\times 1$ convolutional layer is intrinsically similar to the bottleneck architecture that is adopted in all the state-of-the-art CNNs derived from ResNet\cite{he2016deep}. The only difference is that the bottleneck architecture uses one extra $1\times 1$ convolution layer before the $k\times k$ convolution layer. 
This typical bottleneck architecture has been adopted in the state-of-the-art CNN architectures in recent years\cite{ma2024efficient,yu2024inceptionnext,iclr2024MogaNet,lou2025overlock}. This allows our methods to be easily integrated into or combined with these architectures for a potential performance improvement.

Apart from $1\times 1$ layers, we also note that the channel grouping convolution technique\footnote{It should be noted that here the channel group is a concept that differs from the transformation group. The channel grouping convolution technique divides the input feature map channels into multiple channel groups of the same width to perform convolution operations separately.} proposed in ResNeXt\cite{xie2017aggregated} is also a helpful technique for improving CNN's performance.

Thanks to the flexibility of the proposed WMCG-CNN, we can easily combine these techniques with the WMCG-CNN. An example is shown in Fig. \ref{fig:shearfilter}. Similar blocks but with different filter sizes will be used in later image denoising experiments.

\section{Experiments}
We test WMCG-CNN on classification and regression tasks, such as image classification and image denoising. 
In the image classification part, we also conducted ablation experiments, and compared our method with the parameter-sharing group equivariant methods.

\subsection{Performance metrics}

We adopt the following performance metrics:
the number of trainable parameters in million ($10^6$), Params(M); the number of Multiply–Accumulate Operations in giga ($10^9$), MACs(G); the prediction error in percentage, Error(\%); mean prediction error on corrupted validation image datasets in percentage, mCE(\%); top 1 accuracy in percentage, top-1 acc.(\%); top 5 accuracy in percentage, top-5 acc.(\%); peak signal-to-noise ratio in dB, PSNR(dB); the degree of parameter-sharing, MACs$/$Params (G/M).

In addition, for the section of the ablation experiments, similar to \cite{worrall2019deep}, we define mean normalized group-equivariant error (mGE) according to equation \eqref{eq:equivariant}:
\begin{equation}
	mGE = \mathbb{E}(\|\phi(\mathbb{T}_g(f))-\mathbb{T}^{\prime}_g(\phi(f)) \|/\|\phi(\mathbb{T}_g(f))\|)
	\label{eq:gerror}
\end{equation}	 
where for each input image, a random affine transformation $g\in G$ is selected with the shear range of $[-0.5\pi,0.5\pi)$, the scaling range of $[1.0,2.0)$ and rotation angle range of $[-1.0\pi,1.0\pi)$.

\subsection{Ablation experiments}

\subsubsection{Experimental setup}
For ablation experiments, we consider a subset of the ImageNet1k dataset. ImageNet1k has $1.28$ million color images for $\num[group-separator={,}]{1000}$ different classes from WordNet. The validation dataset consists of $\num[group-separator={,}]{50000}$ images. For quick experiments, we extract the first $40$ classes for ablation experiments (i.e., from class $n01440764$ to class $n01677366$), and thus we denote the corresponding datasets as ImageNet40. We scale all the images to $224\times 224$ resolution and normalize images in a classic way. 
The prediction Error (\%) is used to measure the classification performance.

We use ResNet18, ResNet50, and ResNeXt50 \cite{xie2017aggregated} as the baseline networks. We follow the state-of-the-art robust training methods as in \cite{hendrycks2022pixmix}. The neural networks are trained for $90$ epochs with an initial learning rate of $0.01$ following a cosine decay schedule. The Pixmix augmentation technique is used with its default setting as in \cite{hendrycks2022pixmix}. Pixmix uses affine transformations (including translation, rotation, and shear transform) as well as other augmentation methods to generate augmented clean images. 
As for WMCG-CNN, we replace all the hidden non-$1\times 1$ CNN layers with the proposed WMCG-CNN layers.  By default, the size of FB basis is $5\times 5$, the number of basis per filter is $9$ (We have the number of bases per filter $9$ (the first $9$ low frequency Bessel filters), the scaling range is of $[1.0,2.0)$, the rotation angle range of $[-2\pi,2\pi)$, and the shear transform angle range of $[-0.25\pi, 0.25\pi)$. In this work, considering the symmetry of the filter basis, by default, we keep the shear angle of $S_2(r)$ as zero for simplicity.

\subsubsection{Experimental results}

\begin{figure}[t]
	\centering
	\includegraphics[width=0.80\linewidth]{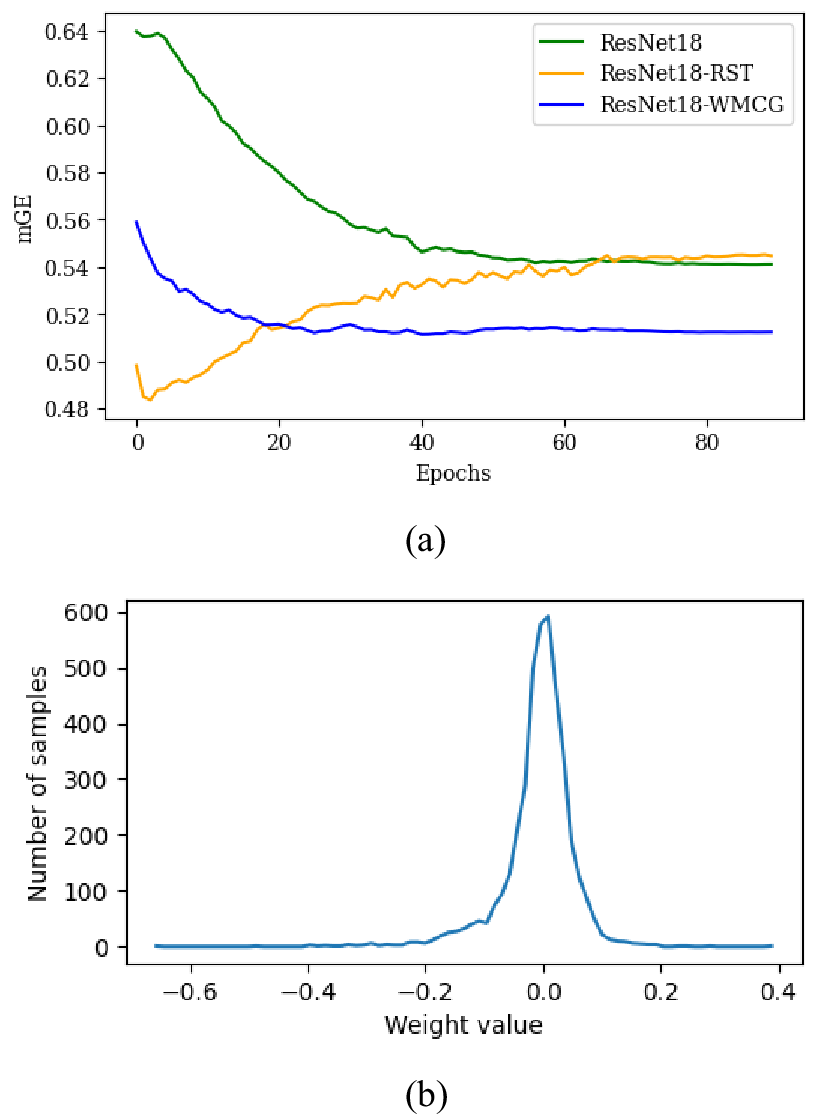}	
	\caption{ (a) The mGEs of the first hidden CNN layer with 256 input and output channels of different residual networks for $\num[group-separator={,}]{90}$ epochs of training on ImageNet dataset. For all the ResNet18 variants, the hidden $3\times 3$ CNNs layers are replaced with a corresponding $5\times 5$ CNNs. The $5\times 5$ CNN layers of ResNet18-RST are RST-CNNs\cite{gao2021deformation}.  (b) The histogram of the learned weights for the FB basis of order $0$ in the first hidden CNN layer of ResNet18-k5-WMCG-shear-0.25$\pi$.}
	\label{fig:mGE}
\end{figure}

Fig. \ref{fig:mGE}a shows the mGE results for the first hidden CNN layer with 256 input and output channels of ResNet18, ResNet18-RST, and ResNet18-WMCG for the $\num[group-separator={,}]{90}$-epoch training on ImageNet dataset. 
We see that compared with the plain CNNs and RST-CNNs, WMCG-CNNs achieve the lowest mGE after 90 epochs of training. The mGE of RST-CNN initially drops sharply but increases gradually to a high value. This implies that the parameter-sharing strategy helps reduce mGE more quickly in the initial stages but may not boost the equivariance in the following training iterations. Figure \ref{fig:mGE}b shows that the distribution of the learned weights is centered around zero.

\begin{figure*}[t]
	\centering
	\includegraphics[width=0.95\linewidth]{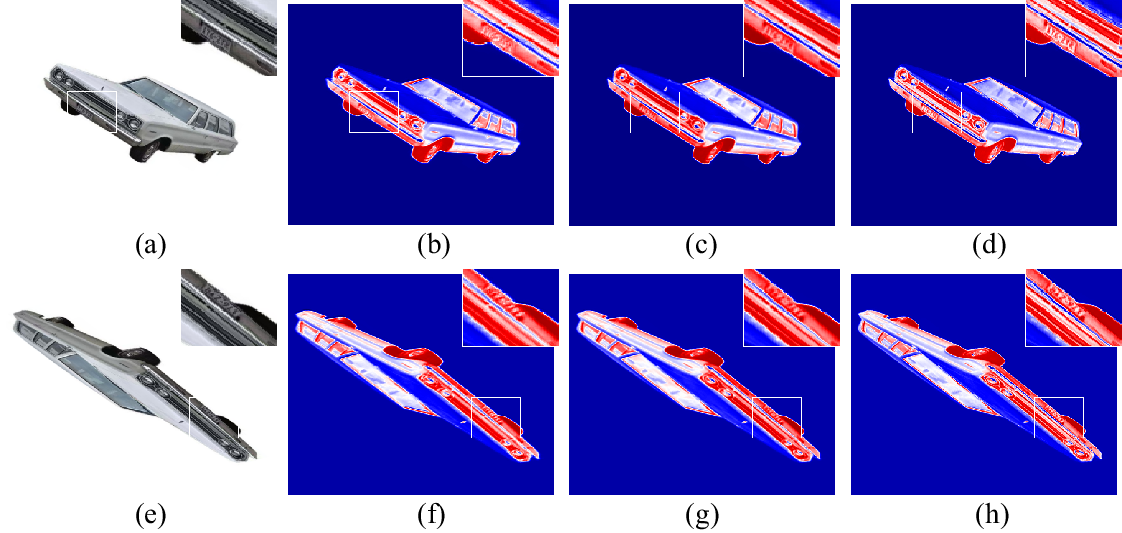}	
	\caption{  The output feature map of CNNs for affinely transformed inputs. The original input image is cropped from a car image selected from STL10\cite{coates2011analysis} dataset. The areas in the white rectangular boxes are magnified to show the details. (a) Input 1; (b) The output of plain CNN for Input 1; (c) The output of RST-CNN for Input 1; (d) The output of WMCG-CNN for Input 1; (e) Input 2; (f) The output of plain CNN for Input 2; (g) The output of RST-CNN for Input 2; (h) The output of WMCG-CNN for Input 2.}
	\label{fig:acti_map}
\end{figure*}

Fig. \ref{fig:acti_map} shows the output feature maps of the plain CNN, RST-CNN\cite{gao2021deformation}, and our WMCG-CNN for affinely transformed input images. All the CNNs consist of the classic two layers structure in bottleneck block of ResNet, where the first convolution layer has a kernel size of $5\times 5$, 3 input channels, and 256 output channels, and the second layer is of kernel size $1\times 1$ with 1 output channel. The $1\times 1$ convolutional layers for all the CNNs have the same weights. 
We see that both RST-CNN and our WMCG-CNN preserve the ID code on the front license plate better than the plain CNN. The ID code in the feature map of plain CNN is heavily distorted (as shown in Fig. \ref{fig:acti_map}(b) and Fig. \ref{fig:acti_map}(f)) along with the affine transformation, which implies its poor affine equivariance.The output features of our WMCG-CNN have clearer edges and richer details than RST-CNN. This is likely because while both networks keeps the same number of channels, the non-parameter-sharing weighted aggregation design of our WMCG-CNN allows it to employ more diverse filters than RST-CNN, which helps the network to capture more image features efficiently.

\begin{table*}[htbp]
	\centering
	\caption{The ablation experiments on the effect of shear augmentation range on the ImageNet40 dataset for different residual networks and filter bases. Dirac means the discrete Dirac delta basis. Kernel Size refers to the size of non-$1\times 1$ kernels of the hidden CNN layers.
	}
	\centerline{\resizebox{12.0cm}{!}{ 
			\begin{threeparttable}[b]
				\begin{tabular}{l|cccc|cc|c} 
					\hline
					Base Model         & Filter  & Kernel Size & Method & Shear Range &   \multicolumn{1}{l}{Params (M)}& \multicolumn{1}{l}{MACs (G)}& \multicolumn{1}{l}{Error (\%)} \\ 
					\hline
					\multirow{5}{*}{ResNet18\cite{he2016deep}} & Dirac & $3\times 3$ &  Vanilla &   & 11.69    &  1.82    & 24.85                   \\ 
					& MH & $3\times 3$ & WMCG & 0.00  &   11.69  &   1.82   & 25.10                 \\ 
					& MH & $3\times 3$ & WMCG & $\pm$0.12$\pi$   &  11.69   &  1.82    & \textbf{24.55}         \\ 
					& MH & $3\times 3$ & WMCG & $\pm$0.25$\pi$  &  11.69   &   1.82   & 26.85                  \\ 
					& MH & $3\times 3$ & WMCG & $\pm$0.50$\pi$  &  11.69   &  1.82    & 25.10                 \\ 
					\hline
					\multirow{4}{*}{ResNet18\cite{he2016deep}}& FB & $5\times 5$ & WMCG & 0.00   &   11.69  &   4.80   & 19.60                 \\ 
					& FB & $5\times 5$ & WMCG & $\pm$0.12$\pi$  &  11.69   &  4.80    & \textbf{18.80}         \\ 
					& FB & $5\times 5$ & WMCG & $\pm$0.25$\pi$   &  11.69   &   4.80   & 19.20                  \\ 
					& FB & $5\times 5$ & WMCG & $\pm$0.40$\pi$   &  11.69   &  4.80    & 19.40                 \\ 
					\hline
					\multirow{5}{*}{ResNeXt50\cite{xie2017aggregated}}& Dirac & $3\times 3$ & Vanilla &    &  25.03   &  4.27     & 27.00                   \\ 
					& FB & $5\times 5$ & WMCG & 0.00    &  25.03   &   4.68   & 27.60                   \\ 
					& FB & $5\times 5$ & WMCG & $\pm$0.12$\pi$   &  25.03   &   4.68   & 27.00                   \\ 
					& FB & $5\times 5$ & WMCG & $\pm$0.25$\pi$  &  25.03   &   4.68   & \textbf{26.95}         \\ 
					& FB & $5\times 5$ & WMCG & $\pm$0.40$\pi$    &   25.03  &  4.68    & 27.90                      \\
					\hline
				\end{tabular}
	\end{threeparttable}}}
	\label{tab:ablation_1}
\end{table*}	

\begin{table*}[htbp]
	\centering
	\caption{The ablation experiments on the choice of Fourier-Bessel bases on the ImageNet40 dataset using ResNet18 with a single Fourier-Bessel basis of kernel size $5\times 5$. The $n$ means the low frequency filter basis with the $n$-th lowest general cutoff frequencies.
	}
	\centerline{\resizebox{6.5cm}{!}{ 
			\begin{threeparttable}[b]
				\begin{tabular}{l|cc|c} 
					\hline
					The $n$-th basis     & \multicolumn{1}{l}{Params (M)}& \multicolumn{1}{l}{MACs (G)}& \multicolumn{1}{l}{Error (\%)} \\ 
					\hline
					1st   &   1.92  &   4.80   & \textbf{28.65}                \\ 
					3rd  &   1.92  &   4.80   & 32.05                \\ 
					8th  &   1.92  &   4.80   & 55.45                \\ 
					\hline
				\end{tabular}
	\end{threeparttable}}}
	\label{tab:ablation_2}
\end{table*}	

\begin{table*}[htbp]
	\centering
	\caption{The ablation experiments on different sampling strategies on the ImageNet40 dataset using ResNet18 with Fourier-Bessel bases of kernel size $5\times 5$. "ES" means equally spaced sampling. Sample Number is the number of transforms. Width Scaling is the coefficient for scaling the channel number per transform sample of CNN layers for maintaining the total width of each CNN layer.
	}
	\centerline{\resizebox{10.5cm}{!}{ 
			\begin{threeparttable}[b]
				\begin{tabular}{l|ccc|cc|c} 
					\hline
					Method & Transform Type & Sample Number &  Width Scaling     & \multicolumn{1}{l}{Params (M)}& \multicolumn{1}{l}{MACs (G)}& \multicolumn{1}{l}{Error (\%)} \\ 
					\hline
					ES & scale&4 & 1  &  11.69   &  18.77    & 19.80                 \\ 
					ES & scale&16& 1/4  &  11.69   &  18.77    & 19.00                 \\ 
					MC & scale&16 & 1/4   &  11.69   &  18.77    & 19.20                 \\ 
					MC & affine&16& 1/4   &  11.69   &  18.77    & \textbf{18.90}                 \\ 
					\hline
					ES & scale&4 & 1/4   &   3.45  &   4.80   & \textbf{24.45}                \\ 
					MC & scale&4 & 1/4  &   3.45  &   4.80   & 24.90                \\ 
					MC & affine&4 & 1/4   &   3.45  &   4.80   & 24.55                \\ 
					\hline
					ES & scale&16 & 1/16   &   1.39  &   4.80   & 35.50                \\ 
					MC & scale&16 & 1/16    &   1.39  &   4.80   & 36.75                \\ 
					MC & affine&16 & 1/16   &   1.39  &   4.80   & \textbf{32.65}                \\ 
					\hline
				\end{tabular}
	\end{threeparttable}}}
	\label{tab:ablation_3}
\end{table*}

Table \ref{tab:ablation_1} shows the results of ablation experiments on the effect of shear transform range on ImageNet40, where the results regarding Params(M) and MACs(G) are also displayed. We see that the shear transform with a suitable range of shear angle is helpful for increasing WMCG-CNN's performance. In all the following experiments, we adopt $n_s =0.25$ for FB basis and $n_s =0.5$ for the MH basis by default if not explicitly stated. It should be noted that the MH basis additionally uses translation augmentation that proves to help increase its performance. 

From Table \ref{tab:ablation_2}, we see that the choice of FB basis affects the prediction performance significantly. Low-frequency basis, i.e., Bessel basis of low order, is shown to be more important than high-frequency basis. Therefore, to select a fixed number of bases, we must include the low-order Bessel basis first. 

As shown in Table \ref{tab:ablation_3}, the conventional scale-equivariant CNN architecture (using ES sampling strategy with Sample Number 4 and Width Scaling 1) has a decent prediction error. However, the computational burden is extremely high. When we try to reduce the computational burden by decreasing the width of the network, the number of trainable parameters is reduced significantly, which leads to poorer prediction performance. The MCG-CNN also has a heavy computational burden and is superior to its corresponding G-CNN when we use a larger number of transformations and more transformation types.

Among the tested ResNet baseline architectures, the results with ResNet18 give the lowest mean error, which indicates that the deeper models such as ResNet50 and ResNeXt50 suffer from overfitting because the number of classes is reduced from 1k to 40. However, the WMCG-CNN can reduce the over-fitting consistently for all the considered baseline models. WMCG-CNN versions of ResNet18 yield the best classification performance. Generally, the results on ImageNet40 demonstrate that, with suitable filter bases, WMCG-CNN is superior to standard CNN in sample efficiency, helps avoid overfitting, and enables quicker convergence.

\subsection{Comparison with the state-of-the-art parameter-sharing G-CNNs}

\subsubsection{Experimental setup}
We test all the group equivariant networks on three common small-scale datasets: Rotated-Scaled-and-Sheared MNIST (RSS-MNIST), CIFAR10 \cite{krizhevsky2009learning}, and STL10\cite{coates2011analysis}. The original MNIST dataset \cite{lecun1998gradient} consists of $\num[group-separator={,}]{70000}$ $28 \times 28$ images of handwritten digits.
Similar to \cite{gao2021deformation}, RSS-MNIST is constructed through randomly rotating (by an angle uniformly distributed on $[0, 2\pi]$), shearing (by an angle uniformly distributed on $[-\pi/4,\pi/4)$ as well as rescaling (by a uniformly random factor from [0.3, 1]) the original MNIST \cite{lecun1998gradient} images. The transformed images are zero-padded back to a size of $28 \times 28$. We upscale the image to $56 \times 56$ for better comparison of the models.
The CIFAR-10 dataset consists of color images of size $32\times 32\times 3$ for 10 classes. There are $\num[group-separator={,}]{50000}$ training images and $\num[group-separator={,}]{10000}$ testing images. 
The STL10 dataset consists of  $\num[group-separator={,}]{13000}$ RGB images of size 96 × 96 belonging to 10 different classes, including $\num[group-separator={,}]{5000}$ images for training set and $\num[group-separator={,}]{8000}$ for the test set.
Similar to \cite{gao2021deformation}, we evaluate different models under both in-distribution (ID) and out-of-distribution (OOD) settings. Specifically, the training set remains unchanged. The ID setting keeps the test set unchanged. The OOD setting augments the test set with uniform random affine transformations with scaling range $[0.9,1.1)$, rotation angle range $[0,0.1\pi)$, and shear transform angle range $[0,0.1\pi)$ along vertical and horizontal directions.

About experiments on RSS-MNIST dataset, following the test procedure for group equivariant networks described in \cite{marcos2018scale,ghosh2019scale,gao2021deformation}, we generate six independent realizations of augmented data. Each of them is split into three parts: $\num[group-separator={,}]{10000}$ images for training, $\num[group-separator={,}]{2000}$ for validation, and $\num[group-separator={,}]{50000}$ for testing.
Adam optimizer is used to train all models for $60$ epochs with the batch size set as $128$. The initial learning rate is $0.01$ and decreases tenfold after $30$ epochs.
We compare methods with the state-of-the-art parameter-sharing group equivariant network, attentive G-CNN\cite{romero2020attentive} and RST-CNN\cite{gao2021deformation}. The implementation of the proposed method uses the same filter basis as RST-CNN. Since the number of their basis is much limited, we use the bootstrap resampling method to obtain enough bases. 
The baseline network is a ResNeXt that is constructed by setting the number of bottleneck blocks of each stage in ResNeXt29\cite{xie2017aggregated,hendrycks2019augmix} as 2, 2, 4 with the group number 64, the group width 1 and the kernel size 15 for the non-$1\times 1$ convolutional kernels. This ResNeXt is called ResNeXt26. We further get ResNeXt32 by increasing the number of bolttleneck blocks of the third stage to 6.
	Accordingly, we get ResNeXt26-WMCG and ResNeXt32-WMCG by replacing all $15\times 15$ hidden convolutional layers of the corresponding ResNeXt with the proposed WMCG convolutional layers using bootstrap resampled $15\times 15$ bases, respectively.

For experiments on CIFAR10 dataset\cite{krizhevsky2009learning}, we use ResNeXt29 \cite{xie2017aggregated} as the baseline network. 
We use the augmented continuous MH wavelet filter as the basis for WMCG CNN. We denote "ResNeXt29-WMCG" as the network created by replacing the $3\times 3$ convolution layer with WMCG CNN of the $3\times 3$ MH basis size and each convolutional filter using $1$ basis. The filter basis is augmented with translation, scaling, and shear transforms. Empirically, only for the experiments with CIFAR10 dataset, we use scaling range $[1.0,1.5)$, while in other experiments we keep $[1.0,2.0)$.
We also include comparison results with attentive G-CNN\cite{romero2020attentive} and efficient group equivariant network\cite{he2021efficient} on CIFAR10 dataset. The networks from \cite{romero2020attentive} are trained in the same way as our methods. Since \cite{he2021efficient} does not publish their code, we simply refer to the results in their paper directly.
All the CNNs are trained with the same training strategy as in \cite{hendrycks2019augmix}. Specifically, all the networks are trained using an initial learning rate of $0.1$ and a cosine learning rate schedule. The optimizer uses stochastic gradient descent with Nesterov momentum and a weight decay of $0.0005$. The input images are first pre-augmented with standard random left-right flipping and cropping, and then the Augmix method\cite{hendrycks2019augmix} is applied with its default settings. Augmix uses affine transformations (including translation, rotation, and shear transforms) as well as other augmentation methods to generate augmented clean images. We repeat the training-and-testing experiments for six rounds and report the mean results $\pm$ standard deviation.

For experiments on STL10 dataset\cite{coates2011analysis}, all the CNNs are trained using an initial learning rate of $0.05$ and batch size 64 for $\num[group-separator={,}]{1000}$ epochs. Following \cite{gao2021deformation}, the optimizer uses stochastic gradient descent with Nesterov momentum and a weight decay of $0.0005$. The learning rate decays by $0.2$ at epoch 300, 400, 600 and 800, respectively. We repeat the experiments 6 times and report the mean results $\pm$ standard deviation. 
We compare methods with the state-of-the-art parameter-sharing group equivariant network, SESN\cite{sosnovik2019scale} and RST-CNN\cite{gao2021deformation} on STL10. The implementation of the proposed method uses the same filter basis as SESN\cite{sosnovik2019scale}. Since the number of their basis is much limited, we use the bootstrap resampling method to obtain enough bases. We use modified versions of ResNeXt50 as baseline networks. Specifically, the width of the bottleneck blocks is doubled. The hidden non-$1\times 1$ convolution layers have a kernel of size $7\times 7$ with the group number 64 and the group width 1. The last 3 bottleneck blocks are removed and the last fully connected layer is therefore adapted with the input channel number the same as that of the output of the nearest remaining bottleneck block. The output channel number of first convolutional layer is increased to 4 times, followed by a $1\times 1$ convolution layer for adaptation of the output channel number. The first pooling layer is removed. We denote the modified ResNeXt50 as "ResNeXt41". Finally, we get "ResNeXt41-WMCG" by replacing all $7\times 7$ hidden convolutional layers of ResNeXt41 with the proposed WMCG convolutional layers using bootstrap resampled $7\times 7$ bases.

\subsubsection{Experimental results}
Table \ref{tab:para_gcnn} shows the results of compared G-CNNs on RSS-MNIST and CIFAR10. We see that the WMCG networks outperform the parameter-sharing networks with much less computational burden while not increasing the computational burden of standard CNNs. The MACs/Params results show that the proposed method has almost the same level of parameter-sharing as the standard CNNs. We also note that FB basis used to perform well on ImageNet datasets but obtained poor performance on CIFAR10 datasets. On the other hand, the wavelet basis which used to work poorly on ImageNet surpasses FB basis on CIFAR10.

\begin{table}[htbp]
	\centering
	\caption{The comparison results of image classification experiments with parameter-sharing group equivariant neural network models on MNIST and CIFAR10 datasets.}
	\centerline{\resizebox{9.0cm}{!}{ 
			\begin{threeparttable}[b]
				\begin{tabular}{l|cc|c|c} 
					\hline
					Model                  &        \multicolumn{4}{c}{RSS-MNIST}                              \\ 
					\hline
					& Params (M) &  MACs (G)  &   MACs/Params (G/M)    &         \multicolumn{1}{c}{Error (\%)}                          \\ 
					\hline
					mnist$\_$CNN$\_$56\cite{gao2021deformation}	   & 0.49   &   0.14    &           0.29               & \multicolumn{1}{c}{8.76$\pm$0.11}      \\ 
					romhog$\_$fa$\_$p4cnn\cite{romero2020attentive}	   & 0.02   &   0.02    &           0.63              & \multicolumn{1}{c}{8.07$\pm$0.17}      \\ 
					RST-CNN\cite{gao2021deformation}& 3.24   &     5.68   &       1.75         &      \multicolumn{1}{c}{4.96$\pm$0.10}   \\
					ResNeXt26    & 3.28 &     1.51    &      0.46          &      \multicolumn{1}{c}{5.02$\pm$0.22}   \\
					ResNeXt26-WMCG    & 3.02&     1.51    &     0.50           &      \multicolumn{1}{c}{4.86$\pm$0.24}   \\
					ResNeXt32    & 4.45 &     1.74   &      0.39         &      \multicolumn{1}{c}{5.03$\pm$0.34}   \\
					ResNeXt32-WMCG    & 4.11 &     1.74    &     0.42           &      \multicolumn{1}{c}{\textbf{4.59}$\pm$0.19}   \\
					\hline
					&             \multicolumn{4}{c}{CIFAR10}                           \\ 
					\hline
					& Params (M) &  MACs (G)   &    MACs/Params (G/M)    &  \multicolumn{1}{c}{Error (\%)}                               \\ 
					\hline
					ALL-CNN-$\alpha_F$-p4m\cite{romero2020attentive}& 1.25   &   2.92  &          2.34          & \multicolumn{1}{c}{6.61$\pm$1.44}           \\ 
					RESNET44-$\alpha_F$p4m\cite{romero2020attentive}& 2.70   &   3.45   &        1.28            & \multicolumn{1}{c}{5.86$\pm$0.34}           \\ 
					p4m-E4R18\cite{he2021efficient}  & 6.00   &  3.87    &     0.65                 & \multicolumn{1}{c}{4.96$\pm$0.16}           \\ 
					ResNeXt29\cite{xie2017aggregated}  & 6.81   &  1.08  &    0.16                 & \multicolumn{1}{c}{4.36$\pm$0.25}           \\ 
					ResNeXt29-WMCG & 4.74  &  1.08  &     0.23            & \multicolumn{1}{c}{\textbf{4.05}$\pm$0.16}  \\ 
					\hline
				\end{tabular}
	\end{threeparttable}}}
	\label{tab:para_gcnn}
\end{table}

Table \ref{tab:stl10} shows the accuracy results of STL10 under the ID and OOD settings. We find that SESN\cite{sosnovik2019scale} achieves the highest accuracy under the ID setting but performs poorly under the OOD setting. Meanwhile, our WMCG-CNN achieves competitive ID accuracy and the highest OOD accuracy with fewer learnable parameters and lower computational complexity, which demonstrates the good generalization ability of our method.

\begin{table*}[htbp]
	\centering
	\caption{The comparison results of image classification experiments with parameter-sharing group equivariant neural network models on STL10 dataset for both in-distribution (ID) and out-of-distribution (OOD) settings. The Training Time (s/epoch), Training VRAM Footprint (MB) and Inference Time (s) are tested on GPU Nvidia A100 80GB.}
	\centerline{\resizebox{18.0cm}{!}{ 
			\begin{threeparttable}[b]
				\begin{tabular}{l|cc|c|ccc|cc} 
					\hline
					& Params (M) &  MACs (G)  &   MACs/Params (G/M)    &  Training Time (s/epoch) & Training VRAM Footprint (MB) &  Inference Time (s) &       \multicolumn{1}{c}{ID Accuracy (\%)}       &         \multicolumn{1}{c}{OOD Accuracy (\%)}          \\ 
					\hline
					SESN\cite{sosnovik2019scale}	   & 10.96   &   227.02    &           28.47   &   89   &     30440   & 48   & \multicolumn{1}{c}{\textbf{91.22}$\pm$0.24} &  82.30$\pm$0.41  \\ 
					RST-CNN\cite{gao2021deformation}& 9.74   &     850.23   &       106.62       &  437    &   79248    &  184 &   \multicolumn{1}{c}{88.81$\pm$0.12}  & 79.35$\pm$0.53 \\
					ResNeXt41    & 8.15 &     10.05    &         1.26       &   27  &   21886   &  11   & \multicolumn{1}{c}{87.93$\pm$0.64} & 81.13$\pm$0.82\\
					ResNeXt41-WMCG    & 7.97 &     10.05    &     1.28           &   31  &    23184  &   11  & \multicolumn{1}{c}{90.78$\pm$0.59} & \textbf{83.71}$\pm$0.45\\
					\hline
				\end{tabular}
	\end{threeparttable}}}
	\label{tab:stl10}
\end{table*}	

\subsection{Image classification experiments on ImageNet benchmark datasets}

\subsubsection{Experimental setup}
In this section, we test the proposed method on ImageNet1k datasets. In addition, we use ImageNet1k-C\cite{hendrycks2019benchmarking} validation datasets to test neural networks' robustness and generalizability against image corruptions, where 15 diverse corruption types\cite{hendrycks2019benchmarking} are included for both the ImageNet1k-C validation datasets. It should be noted that ImageNet1k-C is not included in the training set but serves as out-of-distribution data to test the models' robustness and generalizability.
Three kinds of training routines are used: For tests with ResNets, we have robust training strategies (Pixmix) with affine transform augmentation included,  and the corresponding controlled training routine without Pixmix augmentation; And the state-of-the-art full training strategy for tests with Swin\cite{liu2021swin} and ConvNeXt\cite{liu2022convnet}.

We use ResNeXt50\cite{xie2017aggregated} as the baseline network for the Pixmix-based\cite{hendrycks2022pixmix} robust training. We denote "ResNeXt50-WMCG" as the network created by replacing the $3\times 3$ convolution layer with WMCG-CNN of the $5\times 5$ FB basis size and each convolutional filter using $9$ bases. 
The neural networks are trained with the same strategy in Pixmix\cite{hendrycks2022pixmix}. All the neural networks are trained from scratch to compare the sample efficiency and convergence speed of different networks.

In addition, we test our methods with the recently proposed ConvNeXt network model\cite{liu2022convnet} on ImageNet40 and ImageNet1k datasets. We use ConvNeXt-S as the baseline network. We denote "ConvNeXt-S-WMCG" as the network created by replacing all the $7\times 7$ convolution layer with WMCG-CNN of the $7\times 7$ FB basis size and each convolutional filter using $49$ bases. The training on both datasets is in the same way as described in \cite{liu2022convnet}, where the neural networks are trained for 300 epochs using an AdamW optimizer. Similar to \cite{liu2022convnet}, the top 1 and top 5 accuracies are considered.

Although this paper focuses on equivariant CNNs, we note that the idea of "non-parameter-sharing" "weighted aggregation" is also adopted in other research works for image classification, such as dynamic convolutions\cite{chen2020dynamic}. Our methods differ from dynamic convolutions in three aspects: 1) Dynamic convolutions generate weights via an attention module using a feature tensor as input, where the generated weights change with the changes in feature tensors. Our methods use learnable parameters directly as weights (these weights are learned only from the training phase), and the learned weights remain unchanged during the inference phase for different feature tensors; 2) Dynamic convolutions aggregate multiple learnable kernels, while our methods aggregate multiple augmented filters; 3) Dynamic convolutions focus on increasing the representation capability of neural networks by introducing an attention mechanism to aggregate multiple learnable kernels, and the number of learnable parameters is increased. Meanwhile, our methods focus on efficiently improving the affine equivariance of CNNs without increasing computational complexity of inference; the number of learnable parameters is reduced or remains unchanged. To demonstrate the performance difference, we include the tests of dynamic convolutions in the following experiments, where we denote the dynamic convolution as "DY". And the hidden non-$1\times 1$ CNN layers of the base models are replaced with a corresponding dynamic convolution \cite{chen2020dynamic} layer of the same kernel size.

\subsubsection{Experimental results}
Table \ref{tab:cls} shows all the results for our image classification experiments. We see that with or without the robust training strategies (Pixmix), the proposed WMCG-CNNs reduce the classification errors on both clean and corrupted datasets while using the same or smaller number of parameters. 
It is also noted that a large filter size can help increase the classification precision and robustness of neural networks. 
The WMCG-CNN is good at exploiting large-size filters for boosting the performance.
As for the experiment with ConvNeXt, WMCG-CNN improves ConvNeXt-S on both ImageNet40 and ImageNet1k datasets without increasing the number of parameters and the computational burden. It is also noted that shear transform is also helpful for performance boost under the $300$-epoch full-training routine. Dynamic convolution works well with a kernel size of $3\times 3$, but performs poorly with a kernel size of $7\times 7$. Furthermore, it performs poorly on the ImageNet1k-C dataset, which implies poor generalization ability.

\begin{table*}[htbp]
	\centering
	\caption{The results of image classification experiments with CNN models on multiple benchmark datasets.}
	\centerline{\resizebox{10.0cm}{!}{ 
			\begin{threeparttable}[b]
				\begin{tabular}{l|cc|cc|cc} 
					\hline
					&     &    &    &   & \multicolumn{1}{c}{ImageNet1k} & ImageNet1k-C                           \\ 
					\hline
					With Pixmix & Method & Kernel Size & Params (M) &  MACs (G)   & \multicolumn{1}{c}{Error (\%)}       & mCE (\%)                                  \\ 
					\hline
					\multirow{3}{*}{ResNet50\cite{he2016deep}}  &  Vanilla &  $3\times 3$ & 25.56   &   4.12    & 25.78                         & \multicolumn{1}{c}{54.23}      \\ 
					 &  DY &  $3\times 3$   &  59.83  &  4.19    &       \textbf{24.66}             &   55.34    \\ 
					&  WMCG  &    $5\times 5$   & 25.56   &     7.41   & 25.26                &      \multicolumn{1}{c}{\textbf{53.04}}   \\
					\hline
					\multirow{3}{*}{ResNeXt50\cite{xie2017aggregated}} & Vanilla  &  $3\times 3$ & 25.03  &    4.27    & 23.27                         & \multicolumn{1}{c}{51.16}      \\ 
					 & DY &   $3\times 3$  &  30.55  & 4.31  &         \textbf{22.00}     &    51.57   \\ 
					 & WMCG   &  $5\times 5$ & 25.03 &     4.68    & 23.10                &      \multicolumn{1}{c}{\textbf{50.57}}   \\
					\hline
					Without Pixmix & Method & Kernel Size &Params (M) &  MACs (G)   & \multicolumn{1}{c}{Error (\%)}       & mCE (\%)                                  \\ 
					\hline
					\multirow{3}{*}{ResNet50\cite{he2016deep}}   & Vanilla &  $3\times 3$ & 25.56   &   4.12    & 23.52                         & \multicolumn{1}{c}{61.52}      \\ 
					& DY  &  $3\times 3$ &    59.83  &  4.19    &           \textbf{23.24}              &   62.05   \\ 
					& WMCG &  $5\times 5$  & 25.56   &     7.41   & 23.28                &      \multicolumn{1}{c}{\textbf{61.12}}   \\
					\hline
					\multirow{3}{*}{ResNeXt50\cite{xie2017aggregated}}  & Vanilla & $3\times 3$ & 25.03  &    4.27    & 22.13                &      \multicolumn{1}{c}{59.82}      \\ 
					 & DY  & $3\times 3$  &    30.55  & 4.31      &                   22.09                         & \multicolumn{1}{c}{60.36}     \\ 
					& WMCG  & $5\times 5$   & 25.03 &     4.68    & \textbf{22.01}                &      \multicolumn{1}{c}{\textbf{59.71}}   \\
					\hline
					&      &      & \multicolumn{2}{c}{ImageNet40}                            \\ 
					\hline
					&Method & Kernel Size & Params (M) &  MACs (G)   &   \multicolumn{1}{c}{Top-1 Acc. (\%)} &  \multicolumn{1}{c}{Top-5 Acc. (\%)}      \\ 
					\hline
					\multirow{3}{*}{ConvNeXt-S\cite{liu2022convnet}}  &   Vanilla &  $7\times 7$ & 50.22 &   8.70  & 84.75 & 96.45 \\
					& DY & $7\times 7$  & 53.70 & 8.71 & \textbf{86.85} & 94.50 \\
					 & WMCG  &  $7\times 7$ & 50.22 &   8.70   & 86.65 & \textbf{97.80} \\
					\hline
					&      &      & \multicolumn{2}{c}{ImageNet1k}                            \\ 
					\hline
					& Method & Kernel Size & Params (M) &  MACs (G)    & \multicolumn{1}{c}{Top-1 Acc. (\%)} &  \multicolumn{1}{c}{Top-5 Acc. (\%)}      \\ 
					\hline
					Swin-S\cite{liu2021swin} &  &   & 49.61 &   8.75  & 83.17 & 96.23 \\
					\hline
					\multirow{3}{*}{ConvNeXt-S\cite{liu2022convnet}}  &  Vanilla  &  $7\times 7$ & 50.22 &    8.70   & 83.14 & 96.43 \\
					& DY  & $7\times 7$ & 53.70 & 8.71  & 82.57 & 96.17 \\
					 & WMCG  &  $7\times 7$ & 50.22 &   8.70   &\textbf{83.24} & \textbf{96.49} \\
					\hline
				\end{tabular}
	\end{threeparttable}}}
	\label{tab:cls}
\end{table*}

\subsection{Image denoising experiments on simulation and real noisy image sets}
Although it has been shown that in certain cases with known noise levels, traditional algorithms can surpass CNNs in denoising quality\cite{Zhao8169017,Zhao8719026}, their processing speed is much slower than neural networks. Blind denoising with unknown noise levels is also a more practical scenario in the application. Thus, in this paper, we only test the neural networks' performance on blind denoising tasks. 

The experiments are divided into three parts: grayscale synthetic additive Gaussian noisy image denoising, color synthetic additive Gaussian noisy image denoising, and real-world color noisy image denoising (whose image noise is generated in the camera imaging process).

\subsubsection{Datasets}
For grayscale image denoising, as in \cite{zhang2017beyond}, the same $400$ $180\times 180$ images are used for training. These images are cropped from each of the selected 400 images of the Berkeley segmentation dataset\cite{martin2001database}. The Berkeley segmentation dataset includes $\num[group-separator={,}]{1000}$ representative $481\times 321$ RGB images from the Corel image database. The training images are corrupted by synthetic additive Gaussian noise of noise level (i.e., the standard deviation of noise) $\sigma \in [0,55]$. $128\times \num[group-separator={,}]{3000}$ patches of size $50\times 50$ are cropped to train the CNN model.
For color synthetic noisy image denoising, we follow \cite{tian2021designing}, where the same $400$ color images are augmented with bicubic downscaling, counterclockwise rotation, and horizontal flip.
As for real-world noisy images, as in \cite{tian2021designing}, the training dataset consists of $100$ $512\times 512$ JPEG images collected from five digital cameras Canon 80D, Nikon D800, Canon 600D, Sony A7 II and Canon 5D Mark II with an ISO of 800, \num[group-separator={,}]{1600}, \num[group-separator={,}]{3200}, \num[group-separator={,}]{6400}, \num[group-separator={,}]{12800} and \num[group-separator={,}]{25600}.

Five public test datasets are considered, including the grayscale image datasets Set12\cite{li2013benchmark}, BSD68\cite{li2013benchmark}, the color image datasets CBSD68\cite{li2013benchmark}, Kodak24\cite{franzen1999kodak}, and the public real noisy consumer camera image dataset CC\cite{nam2016holistic}. Set12 and BSD68 consist of 12 and 68 grayscale images derived from the Berkeley segmentation dataset\cite{martin2001database}, respectively. CBSD68 consists of 68 color images selected from the Berkeley segmentation dataset\cite{martin2001database}. The Kodak24 dataset consists of 24  lossless, true color (24 bits per pixel) $ 768\times 512$ images released by the Eastman Kodak Company. The public CC dataset consists of 15 images that are captured by three different digital cameras: Canon 5D Mark III, Nikon D600, and Nikon D800 with ISO values of \num[group-separator={,}]{1600}, \num[group-separator={,}]{3200}, or \num[group-separator={,}]{6400}. The training images are cropped into $41\times 41$ patches for training the networks.

\subsubsection{Experimental setup for grayscale image denoising}
We consider one of the most famous denoising CNNs, DnCNN-B\cite{6247952}\cite{zhang2017beyond} as the baseline network for experiments on gray-scale image denoising. We build a brand new denoising network called DnNeXt-B by replacing every plain hidden CNN layer in DnCNN-B with the bottleneck block shown in Fig. \ref{fig:shearfilter}(b). We further denote "DnNeXt-B-k5-WMCG" as the network created by replacing the hidden $3\times 3$ convolution layer in DnNeXt-B with WMCG-CNN of the $5\times 5$ FB basis size and each convolutional filter decomposed by $9$ bases. Likewise, "DnNeXt-B-k7-WMCG" is a corresponding version with FB basis of size $7\times 7$.
To emphasize the efficiency of our approach, we also include another wavelet-based denoising CNN, MWDCNN\cite{tian2023multi} for comparison. 
We test all the CNNs on the standard grayscale image datasets Set12\cite{li2013benchmark}, and BSD68\cite{li2013benchmark}. 
The DnCNN, DnNeXt, and DnNeXt-WMCG are trained with the same training strategy as in \cite{zhang2017beyond}.
We use the SGD optimizer with a weight decay of $0.0001$, and a momentum of $0.9$. The networks are trained for 50 epochs with a batch size of 128. During the 50 epochs of training, the learning rate decreases exponentially from $1.0\times 10^{-1}$ to $1.0\times 10^{-4}$.

\subsubsection{Experimental results for grayscale image denoising}
Table \ref{tab:noisysim} shows the denoising results with the metric of peak signal-to-noise ratio (PSNR) on images corrupted by simulated white Gaussian noise of different noise levels. The number of trainable parameters and MACs are also displayed. In particular, for all the calculations of MACs in image-denoising experiments, we assume the input patch size is $3\times 32\times 32$ for a fair comparison of computational burden, which differs from the actual case.
We find that the proposed DnNeXt and DnNeXt-MCG outperform DnCNN and MWDCNN with a much smaller number of learnable parameters. In addition, the proposed DnNeXt-WMCG achieves the highest average PSNR of all CNNs and yields especially higher PSNR on high noise levels. The larger FB basis helps gain a higher PSNR score on high noise levels, yet may cause poor performance on low noise levels. 

\subsubsection{Experimental setup for color image denoisng}
We consider DudeNet\cite{tian2021designing}, an upgrading of DnCNN as the baseline CNN for the synthetic color noisy image denoising and real camera image denoising experiment. We build a new network DudeNeXt by replacing every plain hidden $3\times 3$ CNN layer in DudeNet with the bottleneck block shown in Fig. \ref{fig:shearfilter}(b). We further denote "DudeNeXt-k5-WMCG" as the network created by replacing the hidden $3\times 3$ convolution layer in DudeNeXt with WMCG-CNN of the $5\times 5$ FB basis size and each convolutional filter decomposed by $9$ bases. We follow the same training strategy as in \cite{tian2021designing}. We use the Adam optimizer with an initial learning rate of $1.0\times 10^{-3}$ and a batch size of 128. The networks are trained for 70 epochs. During the 70 epochs of training, the learning rate decreases exponentially from $1.0\times 10^{-3}$ to $1.0\times 10^{-5}$. On the synthetic color noisy dataset, we compare our methods with two conventional denoising algorithms, CBM3D \cite{dabov2007image}, TID\cite{luo2015adaptive}, as well as three deep learning methods DnCNN\cite{zhang2017beyond}, DudeNet\cite{tian2021designing}, and MWDCNN\cite{tian2023multi}. 
On the real noisy image set CC, we additionally include two transformer-inspired networks for comparison, i.e., Restormer\cite{zamir2022restormer} and NAFNet-width64\cite{chen2022simple}. For a fair comparison, we keep using the same training set, batch size, and epoch number. For parameter stabilization, these two networks use their best optimizer setting, respectively. Specifically, as in \cite{zamir2022restormer}, Restormer uses the AdamW optimizer with an initial learning rate of 3e-4, betas of (0.9, 0.999), weight decay of 1e-4, and the cosine annealing learning rate scheduler. As in \cite{chen2022simple}, NAFNet uses AdamW optimizer with an initial learning rate of 1e-3, betas of (0.9, 0.9), weight decay of 0, and the cosine annealing learning rate scheduler.

\subsubsection{Experimental results for color image denoising}
Table \ref{tab:colornoisysim} shows the results on the public CBSD68 and Kodak24 color image datasets. Table \ref{tab:real_noisy} shows the results on the public CC dataset. The average PSNR, the parameter number, and MACs are displayed. Fig. \ref{fig:cc_d800iso1600} and Fig. \ref{fig:cc_d800iso6400} show the visual results on the CC dataset. On both synthetic and real-world color image denoising experiments, the proposed networks achieve superior performance regarding the average PSNR. DudeNeXt-WMCG gives a competitive visual performance with a lightweight architecture. The performance of Restormer and NAFNet may be further improved with more advanced training and pretraining techniques as well as a larger iteration number. Yet our focus is to test the methods' data efficiency and parameter efficiency. We see that the proposed method helps achieve a lightweight denoising solution by improving DudeNeXt's performance without increasing the computational burden, which uses far fewer parameters than transformer-inspired networks.
In addition, regarding dynamic convolutions, they introduce more learnable parameters and exhibit inferior denoising performance compared to our method. Similar to the results for image classification, dynamic convolutions work well with a kernel size of $3\times 3$ but perform poorly with a kernel size of $5\times 5$.

\begin{figure}[t]
	\centering
	\includegraphics[width=0.9\linewidth]{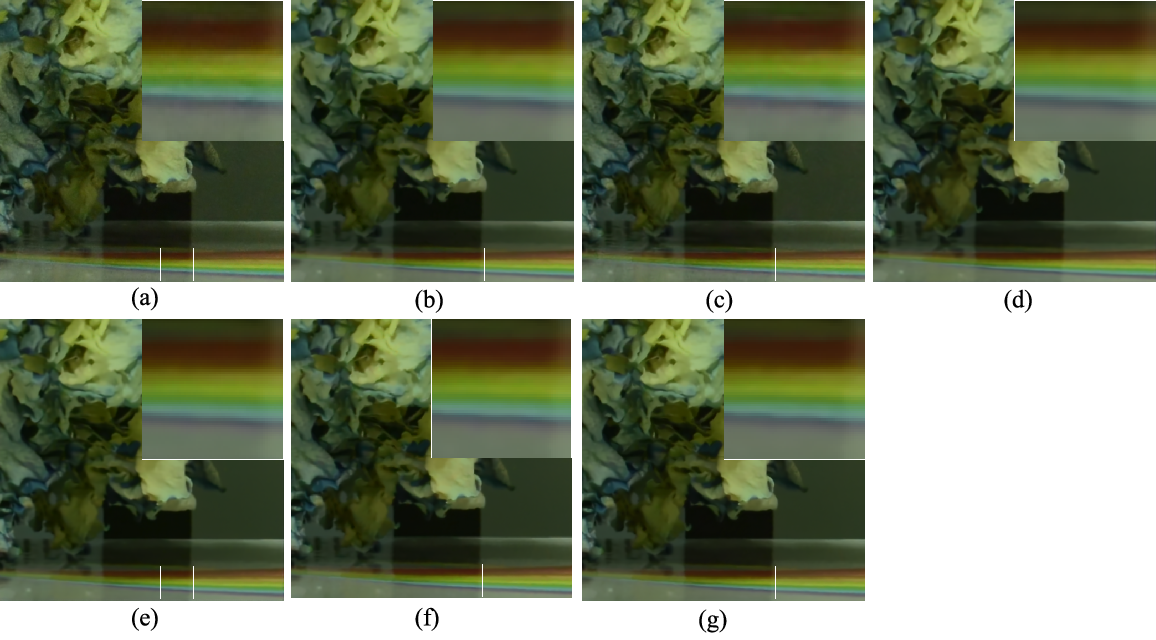}		
	\caption{ Denoising results of a real noisy image by Nikon D800 ISO1600 from CC. (a) noisy image; (b) Dude; (c) MWDCNN; (d) Restormer; (e) NAFNet; (f) DudeNeXt; (g) DudeNeXt-B-k5-WMCG-nb9.}
	\label{fig:cc_d800iso1600}
\end{figure}

\begin{figure}[t]
	\centering
	\includegraphics[width=0.9\linewidth]{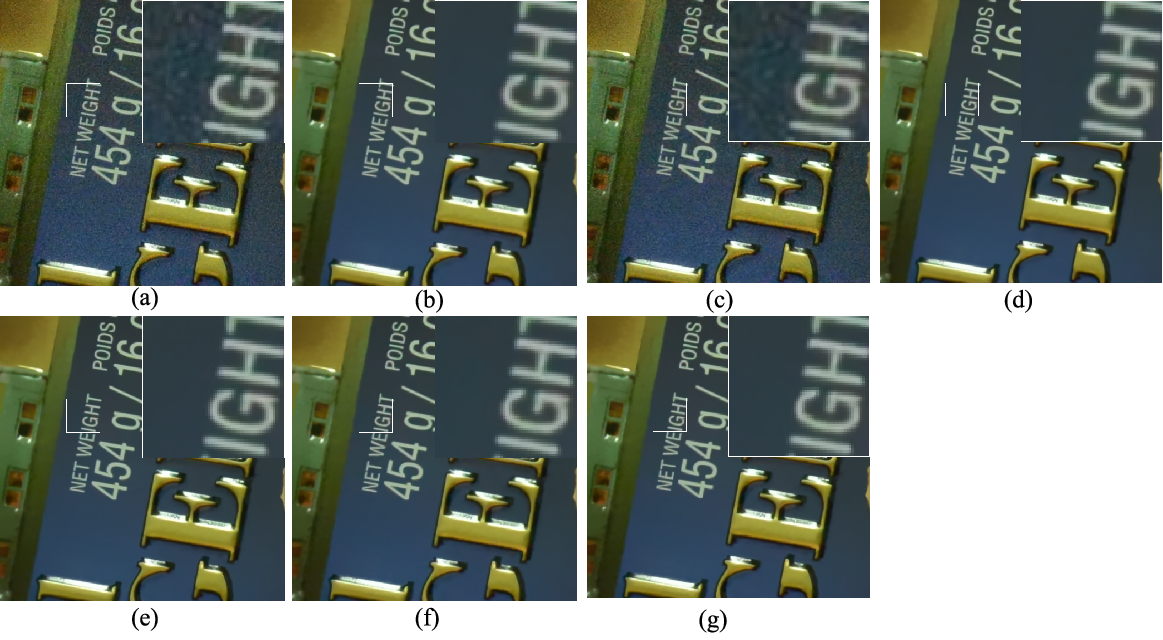}		
	\caption{ Denoising results of a real noisy image by Nikon D800 ISO6400 from CC. (a) noisy image; (b) Dude; (c) MWDCNN; (d) Restormer; (e) NAFNet; (f) DudeNeXt; (g) DudeNeXt-B-k5-WMCG-nb9.}
	\label{fig:cc_d800iso6400}
\end{figure}

\begin{table*}[htbp]
	\centering
	\caption{The average PSNR (dB) of different methods on the grayscale image datasets Set12 and BSD68 with different noise levels $\sigma$ from 15 to 50.}
	\centerline{\resizebox{14.0cm}{!}{ 
			\begin{threeparttable}[b]
				\begin{tabular}{l|cc|cc|cccc|cccc|c} 
					\hline
					& Method & Kernel Size & Params (M)   &  MACs (G)  & \multicolumn{4}{c|}{Set12}                                        & \multicolumn{4}{c|}{BSD68}  &   Average                                                      \\ 
					\hline
					$\sigma$ &     &   &   &     & 15             & 25             & 35             & 50             & 15             & 25             & 35             & 50             &         \\ 
					\hline
					MWDCNN-B\cite{tian2023multi}  & & & 5.24  &  3.75     & 32.60          & 30.39          &                & 27.23          & 31.39          & 29.16          &                & 26.20          &                 \\ 
					DnCNN-B\cite{zhang2017beyond}  & & & 0.67  &    0.68    & 32.70          & 30.35          & 28.78          & 27.13          & 31.60          & 29.14          & 27.65          & 26.19          & 29.19           \\ 
					\hline
					\multirow{3}{*}{DnNeXt-B}       &  Vanilla  &  $3\times 3$       & 0.64   &    0.66     & \textbf{32.76} & 30.38          & 28.86          & 27.18          & \textbf{31.65} & \textbf{29.18} & 27.70          & 26.24          & 29.24           \\ 
					& WMCG  &    $5\times 5$    & 0.64  &    1.26     & 32.74          & \textbf{30.41} & 28.89          & 27.30          & 31.56          & 29.16          & 27.72          & 26.31          & \textbf{29.26}  \\ 
					& WMCG   &    $7\times 7$   & 0.64  &   2.17      & 32.57          & 30.39          & \textbf{28.96} & \textbf{27.37} & 31.21          & 29.06          & \textbf{27.74} & \textbf{26.33} & 29.20           \\
					\hline
				\end{tabular}
	\end{threeparttable}}}
	\label{tab:noisysim}
\end{table*}	

\begin{table*}[htbp]
	\centering
	\caption{The average PSNR (dB) of different methods on the color image datasets CBSD68 and Kodak24  with different noise levels $\sigma$ from 15 to 50.}
	\centerline{\resizebox{14.0cm}{!}{ 
			\begin{threeparttable}[b]
				\begin{tabular}{l|cc|cc|cccc|cccc|c} 
					\hline
					& Method & Kernel Size &Params (M)  &    MACs (G)     & \multicolumn{4}{c|}{CBSD68}                                       & \multicolumn{4}{c|}{Kodak24}              & Average                                \\ 
					\hline
					$\sigma$  &   &  &   &     & 15             & 25             & 35             & 50             & 15             & 25             & 35                    & 50     &         \\ 
					\hline
					CBM3D \cite{dabov2007image} &  &   &    &    & 33.52          & 30.71          & 28.89          & 27.38          & 34.28          & 31.68          & 29.90                 & 28.46  &  30.60       \\ 
					DnCNN \cite{zhang2017beyond} &  &   &  0.56 &   0.57   & 33.98          & 31.31          & 29.65          & 28.01          & 34.73          & 32.23          & 30.64                 & 29.02   &    31.20     \\ 
					FFDNet\cite{zhang2018ffdnet}  &    &    &    0.85  &  0.22  & 33.80          & 31.18          & 29.57          & 27.96          & 34.55          & 32.11          & 30.56                 & 28.99    &   31.09    \\ 
					DudeNet \cite{tian2021designing} &   &    &  1.08  &  1.11   & 34.01          & 31.34          & 29.71          & 28.09          & 34.81          & 32.26          & 30.69                 & 29.10   &  31.25    \\ 
					MWDCNN \cite{tian2023multi} &    &   &  5.25  &   3.76   & 34.18          & 31.45          & 29.81          & 28.13          & 34.91          & 32.40          & 30.87                 & 29.26    &  31.38  \\ 
					MWDCNN-B \cite{tian2023multi} &  &  &   5.25  &   3.76    & 34.10          & 31.44          & 29.80          & 28.15          & 34.83          & 32.39          & 30.83                 & 29.23    &  31.35  \\ 
					DudeNet-B \cite{tian2021designing}  &  &  &  1.08  &  1.11     & 33.96          & 31.32          & 29.69          & 28.05          & 34.71          & 32.23          & 30.66                 & 29.05    &  31.21  \\ 
					\hline
					\multirow{2}{*}{DudeNeXt-B}   &  Vanilla &  $3\times 3$ &    1.07  &    1.04   & 34.15          & 31.46          & 29.80          & 28.12          & 34.90          & 32.40          & 30.81                 & 29.17    &  31.35  \\ 
					 & WMCG & $5\times 5$ &  1.07  &   2.04    & \textbf{34.19} & \textbf{31.53} & \textbf{29.90} & \textbf{28.27} & \textbf{34.96} & \textbf{32.49} & \textbf{30.94}        & \textbf{29.35}  &  \textbf{31.45} \\
					\hline
				\end{tabular}
	\end{threeparttable}}}
	\label{tab:colornoisysim}
\end{table*}

\begin{table*}[htbp]
	\centering
	\caption{The PSNR (dB) of different methods on the real-world noisy image dataset CC\cite{nam2016holistic} by customer cameras.}
	\centerline{\resizebox{19.0cm}{!}{ 
			\begin{threeparttable}[b]
\begin{tabular}{l|cc|ccc|ccc|ccc|ccc|ccc|c|cc}
	\hline
	& Method & Kernel Size                                   & \multicolumn{3}{c|}{Canon 5D ISO3200}             & \multicolumn{3}{c|}{Nikon D600 ISO3200}           & \multicolumn{3}{c|}{Nikon D800 ISO1600}           & \multicolumn{3}{c|}{Nikon D800 ISO3200}           & \multicolumn{3}{c|}{Nikon D800 ISO6400}           & Average        & Params (M) & MACs (G) \\\hline
	CBM3D \cite{dabov2007image}    &   &    & \textbf{39.76} & 36.40          & \textbf{36.37} & 34.18          & 35.07          & 37.13          & 36.81          & 37.76          & 37.51          & 35.05          & 34.07          & 34.42          & 31.13          & 31.22          & 30.97          & 35.19          &            &          \\
	TID\cite{luo2015adaptive}      &   &        & 37.22          & 34.54          & 34.25          & 32.99          & 34.20          & 35.58          & 34.49          & 35.19          & 35.26          & 33.70          & 31.04          & 33.07          & 29.40          & 29.86          & 29.21          & 33.36          &            &          \\
	DnCNN\cite{zhang2017beyond}      &   &      & 37.26          & 34.13          & 34.09          & 33.62          & 34.48          & 35.41          & 37.95          & 36.08          & 35.48          & 34.08          & 33.70          & 33.31          & 29.83          & 30.55          & 30.09          & 33.86          & 0.56       & 0.57     \\
	DudeNet\cite{tian2021designing}   &   &     & 36.66          & 36.70          & 35.03          & 33.72          & 34.70          & 37.98          & 38.10          & \textbf{39.15} & 36.14          & 36.93          & 35.80          & 37.49          & 31.94          & 32.51          & \textbf{32.91} & 35.72          & 1.08       & 1.11     \\
	MWDCNN\cite{tian2023multi}      &   &       & 36.97          & 36.01          & 34.80          & 33.91          & 34.88          & 37.02          & 37.93          & 37.49          & \textbf{38.44} & 37.10          & \textbf{36.72} & 37.25          & 32.24          & 32.56          & 32.76          & 35.74          & 5.25       & 3.76     \\
	Restormer\cite{zamir2022restormer}    &   &     & 35.71          & 35.97          & 34.71          & 33.91          & 35.36          & 38.58          & 37.69          & 37.64          & 36.18          & \textbf{37.96} & 35.88          & 37.84          & 32.98          & 32.60          & 32.85          & 35.72          & 26.11      & 3.41     \\
	NAFNet\cite{chen2022simple}     &   &       & 35.99          & 36.96          & 35.16          & 32.81          & 35.01          & \textbf{39.00} & 37.29          & 38.65          & 36.93          & 38.34          & 36.33          & \textbf{38.43} & 32.31          & 32.67          & 32.22          & 35.87          & 115.98     & 1.01     \\
	\hline
	\multirow{5}{*}{DudeNeXt}                 &   Vanilla   &           $3\times 3$                   & 36.69          & 36.45          & 35.00          & 33.67          & 34.52          & 37.78          & 38.20          & 38.51          & 37.07          & 37.24          & 36.29          & 37.80          & 32.32          & 32.19          & 32.55          & 35.75          & 1.07       & 1.04     \\
	   &            Vanilla          &       $5\times 5$        & 36.65          & 37.02          & 35.22          & 34.05          & 34.56          & 38.58          & 38.12          & 38.58          & 36.80          & 36.95          & 35.83          & 37.17          & 32.34          & 32.62          & 32.81          & 35.80          & 1.99       & 2.04     \\
	    &    DY      &    $3\times 3$     & 36.00          & 36.82          & 35.40          & 33.65          & \textbf{36.55} & 37.74          & 38.10          & 38.20          & 36.14          & 37.71          & 36.51          & 38.27          & 31.94          & \textbf{32.78} & 32.91          & 35.91          & 2.77       & 1.10     \\
	&         DY        &        $5\times 5$       & 36.49          & 34.01          & 33.62          & 33.86          & 34.33          & 35.36          & 35.18          & 35.22          & 34.66          & 33.43          & 32.62          & 32.89          & 29.76          & 30.13          & 30.05          & 33.44          & 6.57       & 2.07     \\
	&     WMCG      &   $5\times 5$     & 36.91          & \textbf{37.11} & 35.15          & \textbf{34.46} & 35.64          & 38.89          & \textbf{38.30} & 38.81          & 37.30          & 37.72          & 35.99          & 37.76          & \textbf{32.36} & 32.77          & 32.90          & \textbf{36.14} & 1.07       & 2.04    \\\hline
\end{tabular}
	\end{threeparttable}}}
	\label{tab:real_noisy}
\end{table*}

\subsection{Analysis and discussion}

The ablation experiments on ImageNet40 demonstrate the sample efficiency of WMCG-CNN for all the tested baseline network architectures, including ResNet18, ResNet50, and ResNeXt50. We note that the proposed method gives a larger reduction in the Error for ResNet18 and ResNet50 than that for ResNeXt50. This is probably because a larger proportion of learnable parameters in ResNeXt50 lies in $1\times 1$ Conv layers which, as shown in the results, causes heavy overfitting on the small dataset ImageNet40.

The comparison experiments with discrete G-CNN, MCG-CNN, and WMCG-CNN on ImageNet prove that the diversity of transformations is helpful for a performance boost. 
Introducing MC sampling allows us to consider any mix of affine transforms. In the experiments on ImageNet40, we see that the additional use of shear transform with a suitable shear range can consistently improve image classification. Meanwhile, a high degree of shear transform can harm the performance because, in discrete implementation, shear transform leads to compression of information along a certain direction that causes information loss.

As we know a rotation transform can be decomposed into three shear transforms\cite{andres1996quasi} as shown in equation \eqref{eq:3shear}, by combining one shear transform and one rotation transform, we can get all the possible rotation and shear transform (along horizontal or vertical direction), which thereby greatly increase the diversity of the considered transforms. In addition, shear transforms are common in daily life (as shown in Fig. \ref{fig:shear_cbsd}). Obviously, a good machine learning model for classification and regression tasks usually should consider as many such kinds of group equivariance as possible.  
\begin{equation}
	\begin{array}{l}
		R(\theta) = 
		\begin{bmatrix}
			\cos{\theta} & \sin{\theta}\\
			-\sin{\theta} & \cos{\theta}
		\end{bmatrix}
		=  \begin{bmatrix}
			1 & -\tan{\frac{\theta}{2}}\\
			0 & 1
		\end{bmatrix}\\
		\cdot 
		\begin{bmatrix}
			1 & 0\\
			\sin{\theta} & 1
		\end{bmatrix}
		\cdot
		\begin{bmatrix}
			1 & -\tan{\frac{\theta}{2}}\\
			0 & 1
		\end{bmatrix}
	\end{array}
	\label{eq:3shear}
\end{equation}	
On the other hand, the shear-transform-augmented convolutional filters can be considered as an example of the classic continuous shear wavelet\cite{antoine1993image,guo2006sparse}. The shear wavelet can help achieve a highly sparse representation of multidimensional data\cite{guo2006sparse}, which also explains the superior performance it brings to the proposed WMCG-CNN. 

\begin{figure}[t]
	\centering
	\includegraphics[width=0.9\linewidth]{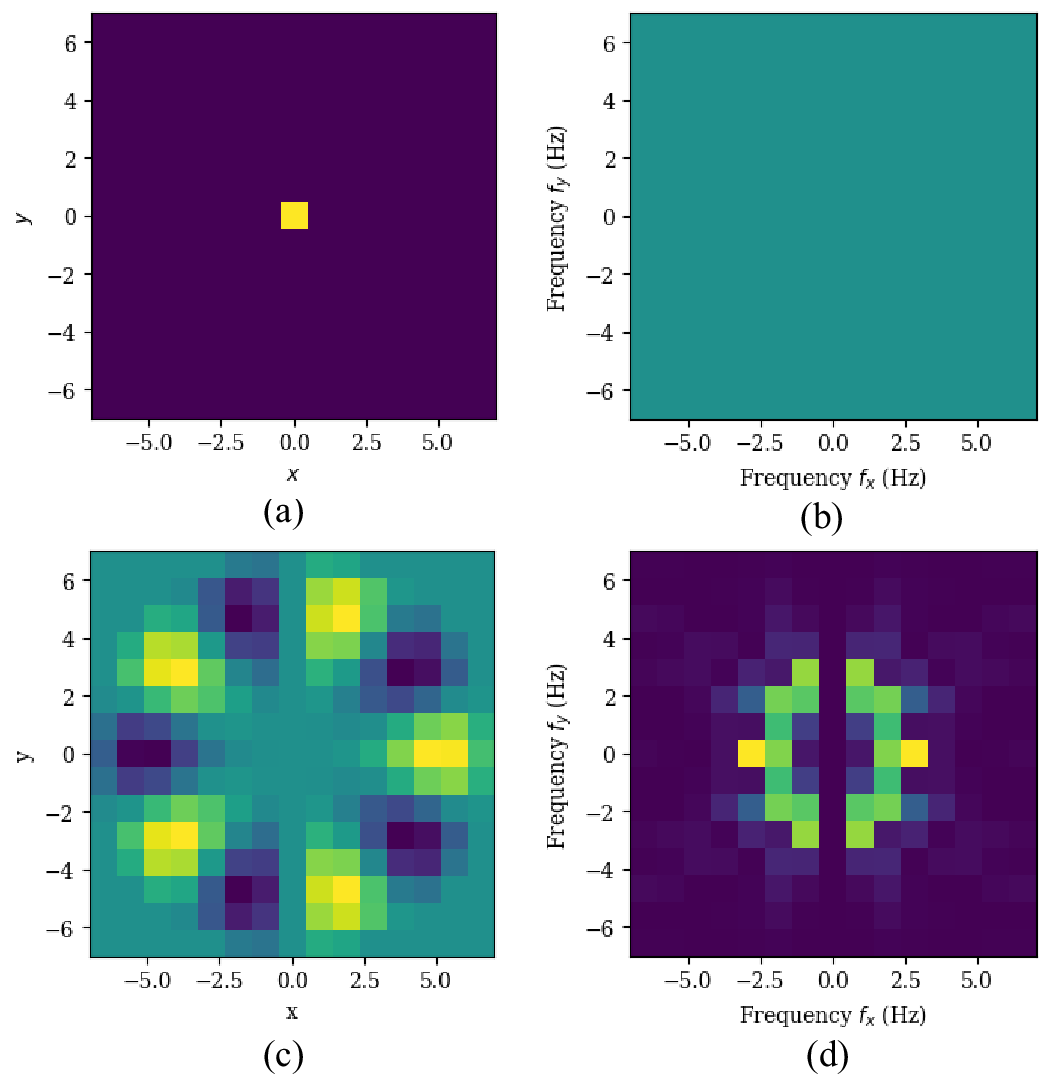}		
	\caption{ The FFT (fast Fourier transform) frequency spectrum of different filter bases. The filter bases are of size $15\times 15$ (a) The discrete Dirac delta basis; (b) The amplitude of the FFT frequency spectrum for (a) ; (c) A typical Fourier-Bessel basis; (d) The amplitude of the FFT frequency spectrum for (c).}
	\label{fig:spectrum}
\end{figure}

For the image denoising task, we did not manage to improve CNN's performance against Dirac delta basis on SIDD datasets\cite{abdelhamed2018high}. 
This is likely because, as shown in \cite{zamir2022restormer,chen2022simple}, compared with results on CC dataset, the denoised images of SIDD dataset are much closer to noise-free images, which means they are less blurred and contains more high frequency information including clearer edges. Meanwhile, as shown in Fig. \ref{fig:spectrum}, the chosen filter bases are all low-pass filters. As suggested by \cite{qiu2018dcfnet,gao2021deformation}, the low-frequency basis can boost the stability of the convolutional kernel. However, compared with the Dirac-delta basis, the low-pass filter bases are inferior in capturing and representing high-frequency information such as edges. This explains the inferior performance of the chosen filter bases on SIDD dataset. A possible solution is that we can combine the non-Dirac-delta basis with the Dirac-delta basis into a certain architecture to achieve superior performance. Since it is not the focus of this paper, we will not discuss it here.

We also note that in MC integral and stochastic simulation, there are a lot of advanced techniques such as quasi-MC sampling\cite{caflisch1998monte}, Markov chain MC\cite{cowles1996markov}, and multi-level MC\cite{giles2015multilevel}. There is a potential that these methods can help improve both MCG-CNN and WMCG-CNN further, and we will study this in future work.

The proposed WMCG-CNN shows higher flexibility and controllability than conventional CNNs. The use of filter decomposition decouples the relationship between the filter size and the number of trainable parameters. For a certain convolutional kernel, the corresponding number of trainable parameters can be as small as only 1, or as large as any integer. In addition, we can choose a certain custom design basis as one likes to control the performance of the network. For example, in the experiment on the CIFAR10 dataset, we simply choose a single low-frequency wavelet basis and can still get a good result on the CIFAR10 dataset. 
It is noted that the choice of filter basis can greatly affect the performance of WMCG-CNN. For different datasets, the optimal basis type and basis size can be different.

\section{Conclusion}
In this paper, we propose an efficient and flexible implementation of group-equivariant CNN based on filter-wise weighted Monte Carlo sampling, which allows a higher degree of diversity of transformations for a performance boost. Compared with parameter-sharing G-CNN, with a suitable filter basis, the proposed non-parameter-sharing WMCG-CNN can exploit deeper neural network architectures without causing heavy computational burden and achieves superior performance. The proposed WMCG-CNN is shown to be an efficient generalization of standard CNN. The utility of diverse transformations for tasks on natural images is demonstrated. The proposed WMCG-CNN shows superior efficiency on both image classification and image denoising tasks when using a suitable set of filter bases. It is possible to extend it for other computer vision tasks such as image segmentation and image reconstruction. However, the choice of filter basis is a key point for yielding good performance with the proposed method, which will be studied in future work.

\section*{Acknowledgments}
This work was supported by the Deutsche Forschungsgemeinschaft (DFG) under grant no. 428149221, by Deutsches Zentrum für Luft- und Raumfahrt e.V. (DLR), Germany under grant no. 01ZZ2105A and no. 01KD2214, and by Fraunhofer Gesellschaft
e.V. under grant no. 017-100240/B7-aneg. The authors would like to thank Ulrike Attenberger for her contribution to funding acquisition.

\end{document}